\pdfoutput=1
\documentclass[wcp]{jmlr}

\usepackage{longtable}

\usepackage{booktabs}
\usepackage{times}
\usepackage{xcolor}
\usepackage{soul}
\usepackage[utf8]{inputenc}
\usepackage[small]{caption}
\usepackage{algorithm}
\usepackage{algorithmic}
\usepackage{graphicx}
\usepackage{courier}
\usepackage{setspace}
\usepackage{multirow}
\usepackage{booktabs}
\usepackage{setspace}
\usepackage{enumitem}
\usepackage{url}
\setitemize{noitemsep,topsep=0pt,parsep=0pt,partopsep=0pt}

\newtheorem{theoremn}{Theorem}
\newtheorem{lemman}{Lemma}

\jmlrvolume{80}
\jmlryear{2018}
\jmlrworkshop{ACML 2018}

\title[Fast Randomized PCA for Sparse Data]{Fast Randomized PCA for Sparse Data}

   \author{\Name{Xu Feng} \Email{fx17@mails.tsinghua.edu.cn}\\
   \Name{Yuyang Xie} \Email{xyy18@mails.tsinghua.edu.cn}\\
   \Name{Mingye Song} \Email{songmy16@mails.tsinghua.edu.cn}\\
   \Name{Wenjian Yu}\thanks{Corresponding author. This work is supported by the National Nature Science Foundation of China under Grant 61872206.} \Email{yu-wj@tsinghua.edu.cn}\\
   \Name{Jie Tang} \Email{jietang@tsinghua.edu.cn}\\
   \addr  BNRist, Department of Computer Science and Technology, Tsinghua University, Beijing 100084, China
  }

\editors{Jun Zhu and Ichiro Takeuchi}

\begin{document}

\maketitle

\begin{abstract}
Principal component analysis (PCA) is widely used for dimension reduction and embedding of real data in social network analysis, information retrieval, and natural language processing, etc. In this work we propose a fast randomized PCA algorithm for processing large sparse data. The algorithm has similar accuracy to the basic randomized SVD (rPCA) algorithm  \citep{Halko2011Finding}, but is largely optimized for sparse data. It also has good flexibility to trade off runtime against accuracy for practical usage. Experiments on real data show that the proposed algorithm is up to 9.1X faster than the basic rPCA algorithm without accuracy loss, and is up to 20X faster than the \texttt{svds} in Matlab with little error. The algorithm computes the first 100 principal components of a large information retrieval data with 12,869,521 persons and 323,899 keywords in less than 400 seconds on a 24-core machine, while all conventional methods fail due to the out-of-memory issue.

\end{abstract}

\section{Introduction}

In machine learning applications, principal component analysis (PCA) is widely used for dimension reduction of input data. It often behaves as a preprocessing step for no matter supervised or unsupervised learning methods. For the problems of social network analysis, information retrieval, natural language processing (NLP), and even recommender system, where input data matrix is usually a sparse one, PCA or the equivalent truncated SVD is also applied. For example, the latent semantic analysis (LSA) \citep{1990Indexing}, which builds dense representation of document in NLP, includes PCA as a major step. However, the application of PCA to real-world problems which require processing large data accurately, often costs prohibitive computational time. Accelerating the PCA for large sparse data is of an absolute necessity. 

The standard method for performing PCA is calculating truncated singular value decomposition (SVD). For sparse matrix, this is usually implemented with $\texttt{svds}$ in Matlab \citep{svds}, or $\texttt{lansvd}$ in PROPACK \citep{propack} which is an accelerated version of $\texttt{svds}$. However, if the dimensions of matrix are large and more than dozens of principle components/directions are needed, these conventional methods would induce large computational expense or even fail due to excessive memory cost. An alternative is the randomized method for PCA, which has gained a lot of attention in recent years. The idea of randomized matrix method is mainly using random projection to identify the subspace capturing the dominant actions of a matrix \citep{Halko2011Finding,yu2018efficient}. Then, a near-optimal low-rank decomposition of the matrix can be computed, so that we can further obtain an approximate PCA. A comprehensive presentation of the relevant techniques and theories is in \cite{Halko2011Finding}. This randomized technique has been extended to compute PCA of data sets that are too large to be stored in RAM \citep{yu2017single}, or to speed up the distributed PCA \citep{Woodruff2014}. For general SVD or PCA computation, the approaches based on it have also been proposed \citep{rsvdpack,alg971}. They outperform the conventional techniques for calculating a few of principle components. Recently, a compressed SVD (cSVD) algorithm was proposed in \cite{Erichson_2017_ICCV}, which is based on a variant of the method in \cite{Halko2011Finding} but runs faster for image and video processing applications. Another idea for computing PCA of large data is performing eigenvalue decomposition to the product of the data matrix's transpose and itself. However, it only has benefit while handling low-dimensional data (less than several thousands in dimension). 

Although there are a lot of work on randomized PCA techniques, they are mostly for processing dense data. Compared with the deterministic methods, they involve the same or fewer floating-point operations (\emph{flops}), and are more efficient for large high-dimensional dense data by exploiting modern computing architectures. While for large sparse data in real world, these benefits may not exist. Investigating the randomized PCA technique for large sparse data and comparing it with other existing techniques are of great interest.

In this work, we first analyze the adaptability of some acceleration skills for the basic randomized PCA (rPCA) algorithm to sparse data, followed by theoretical proofs and computational cost analysis. Then, we propose a modified power iteration scheme which allows odd number of passes over data matrix and thus provides more flexible trade-off between runtime and accuracy. We also devise a technique to efficiently handle the data matrix with more columns than rows, which is ignored in existing work. To wrap them up, we propose a fast randomized PCA algorithm for sparse data (\emph{frPCA}) and its variant algorithm \emph{frPCAt}, suitable for the data matrices with more columns and more rows, respectively. Theoretical analysis is performed to reveal how the efficiency of the proposed algorithms varies with the sparsity of data, the power iteration parameter, and the number of principle components wanted. In the section of experimental results, we first validate the accuracy and efficiency of the proposed algorithms with some synthetic data. The results show that it is up to 9.1X faster than the basic rPCA algorithm and 20X faster than \texttt{svds}, with negligible loss of accuracy. Then, real large data in social network, information retrieval and recommender system problems are tested. The results show the proposed algorithm is up to 8.7X faster than the basic rPCA. And, it successfully handles the largest case in less than 400 seconds and with 23 GB memory, while the \texttt{svds} fails due to out-of-memory issue (requesting more than 32 GB memory).

For reproducibility, the codes and test data in this work will
be shared on GitHub (\url{https://github.com/XuFengthucs/frPCA_sparse}).

\section{Preliminaries}
In algorithm description, the Matlab conventions are used for specifying row/column indices of a matrix and some operations on sparse matrix.
\subsection{Singular Value Decomposition and PCA}
A standard method for performing PCA is to calculate truncated SVD. The SVD of $\mathbf{A}\in\mathbb{R}^{m\times n}$ is:
\begin{equation}
\mathbf{A}=\mathbf{U\Sigma V}^{\mathrm{T}},
\end{equation}
where $\mathbf{U}=[\mathbf{u}_1, \mathbf{u}_2, \cdots]$ and $\mathbf{V}=[\mathbf{v}_1, \mathbf{v}_2, \cdots]$ are orthogonal matrices which represent the left and right singular vectors, respectively. The diagonal matrix $\mathbf{\Sigma}$ contains the singular values $(\sigma_1, \sigma_2, \cdots)$ of $\mathbf{A}$ in descending order. Suppose that $\mathbf{U}_k$ and $\mathbf{V}_k$ are matrices with the first $k$ columns of $\mathbf{U}$ and $\mathbf{V}$ respectively, and $\mathbf{\Sigma}_k$ is the diagonal matrix containing the first $k$ singular values of $\mathbf{A}$. Then, the truncated SVD of $\mathbf{A}$ can be represented as:
\begin{equation}
\mathbf{A}\approx \mathbf{A}_k = \mathbf{U}_k\mathbf{\Sigma}_k\mathbf{V}_k^{\mathrm{T}}.
\end{equation}
Notice that $\mathbf{A}_k$ is the best rank-$k$ approximation of the initial matrix $\mathbf{A}$ in either spectral norm of Frobenius norm \citep{eckart1936}.

The approximation properties of SVD explain the  equivalence between the truncated SVD and PCA. Suppose each row of matrix $\mathbf{A}$ is an observed data. The matrix is assumed to be
centered, i.e., the mean of all rows is a zero row vector. Then, the leading left singular vectors {$\mathbf{u}_i$} are the principal components. Particularly, $\mathbf{u}_1$ is the first principal component.

The built-in function \texttt{svds} in Matlab is a common choice to compute truncated SVD. It is based on a Krylov subspace iterative method and is especially efficient for sparse matrix. For a dense matrix $\mathbf{A}\in\mathbb{R}^{m\times n}$, \texttt{svds} costs $O(mnk)$ \emph{flops} for computing rank-$k$ truncated SVD. The cost becomes $O(nnz(\mathbf{A})k)$ \emph{flops} when $\mathbf{A}$ is sparse, where $nnz(\cdot)$ means the number of nonzero elements. $\texttt{lansvd}$ in PROPACK \citep{propack} is also an efficient program, written in Matlab/Fortran, for computing the dominant singular values/vectors of a sparse matrix. $\texttt{lansvd}$ can cost two to three times less CPU time than \texttt{svds}. However, there is no parallel version of $\texttt{lansvd}$, so that its actual runtime on a modern computer is often longer than that of \texttt{svds}.

\subsection{The Basic Randomized Algorithm for PCA}
Previous work has shown that the randomized methods have advantages for solving the least linear squares problem and low-rank matrix approximation \citep{drineas2016randnla}. The method for low-rank approximation mainly relies on the random projection to identify the subspace capturing the dominant actions of matrix $\mathbf{A}$, which can be realized by multiplying $\mathbf{A}$ with a random matrix on its right or left side to obtain the subspace's orthogonal basis matrix $\mathbf{Q}$. Then, the low-rank approximation  in form of $\mathbf{QB}$ is computed and further results in
 the approximate truncated SVD \citep{Halko2011Finding}. Because $\mathbf{Q}$ has much fewer columns than $\mathbf{A}$, this method reduces the computational time. It derives the basic randomized PCA (rPCA) algorithm, described as Algorithm 1.
\begin{algorithm}
    \caption{basic rPCA}
    \label{alg1}
    \begin{algorithmic}[1]
      \REQUIRE $\mathbf{A}\in\mathbb{R}^{m\times n}$, rank parameter $k$, power parameter $p$
      \ENSURE $\mathbf{U}\in\mathbb{R}^{m\times k}$, $\mathbf{S}\in\mathbb{R}^{k\times k}$, $\mathbf{V}\in\mathbb{R}^{n\times k}$
      \STATE $\mathbf{\Omega} = \mathrm{randn}(n, k+s)$
      \STATE $\mathbf{Q} = \mathrm{orth}(\mathbf{A}\mathbf{\Omega})$
      \FOR {$i=1, 2, \cdots, p$}
        \STATE $\mathbf{G} = \mathrm{orth}(\mathbf{A}^{\mathrm{T}}\mathbf{Q})$
        \STATE $\mathbf{Q} = \mathrm{orth}(\mathbf{A}\mathbf{G})$
      \ENDFOR
      \STATE $\mathbf{B}  = \mathbf{Q}^{\mathrm{T}}\mathbf{A}$
      \STATE $[\mathbf{U}, \mathbf{S}, \mathbf{V}] = \mathrm{svd}(\mathbf{B})$
      \STATE $\mathbf{U} = \mathbf{Q}\mathbf{U}$
      \STATE $\mathbf{U} = \mathbf{U}(:, 1:k), \mathbf{S} = \mathbf{S}(1:k, 1:k), \mathbf{V} = \mathbf{V}(:, 1:k)$
    \end{algorithmic}
  \end{algorithm}

In Alg. 1, $\mathbf{\Omega}$ is a Gaussian i.i.d matrix. Although other kinds of random matrix can replace $\mathbf{\Omega}$ to reduce the computational cost of $\mathbf{A\Omega}$, they bring some sacrifice on accuracy. The $s$ in Alg. 1 is an oversampling parameter which enables $\mathbf{\Omega}$ with more than $k$ columns for better accuracy. $s$ is a small integer such as 5 or 10.
With $\mathbf{Q}$, we have $\mathbf{A}\approx \mathbf{QB}=\mathbf{QQ}^{\mathrm{T}}\mathbf{A}$. By performing the economic SVD on the $(k+s)\times n$ matrix $\mathbf{B}$ the approximate truncated SVD of $\mathbf{A}$ is obtained. To improve the accuracy of the approximation, power iteration (PI) scheme can be applied \citep{Halko2011Finding}, i.e., Steps 3$\sim$6. It is based on that matrix $(\mathbf{AA}^{\mathrm{T}})^p\mathbf{A}$ has exactly the same singular vectors as $\mathbf{A}$, but its singular values decay more quickly. Therefore, performing the randomized QB procedure on $(\mathbf{AA}^{\mathrm{T}})^p\mathbf{A}$ can achieve better accuracy. The orthonormalization operation ``orth($\cdot$)'' is used to alleviate the round-off error in the floating-point computation. It can be implemented with a call to a packaged QR factorization (e.g., \texttt{qr(X, 0)} in Matlab).

The basic rPCA algorithm with the PI scheme has the following guarantee \citep{Halko2011Finding,musco2015}:
\begin{equation}
||\mathbf{A}-\mathbf{QQ}^\mathrm{T}\mathbf{A}||=||\mathbf{A}-\mathbf{USV}^\mathrm{T}||\le (1+\varepsilon)||\mathbf{A}-\mathbf{A}_k||,
\end{equation}
with a high probability. Here, $\mathbf{A}_k$ is the best rank-$k$ approximation of $\mathbf{A}$. 

Assuming that multiplying an $m\times n$ sparse matrix $\mathbf{A}$  and a $n\times l$ dense matrix  costs $C_{mul}nnz(\mathbf{A})l$  \emph{flops},  performing QR factorization  of  an $m\times n$ matrix costs $C_{qr}mn \min(m, n)$ \emph{flops} and performing the economic SVD on an $m\times n$ matrix costs $C_{svd}mn\min(m,n)$ \emph{flops}, we can analyze the \emph{flop} count of  Alg. 1. Using $\mathrm{FC}_1$ to denote it, we have:
\begin{equation}
\begin{aligned}
\mathrm{FC}_{1}=pC_{qr}nl^{2}+(p+1)C_{qr}ml^2+(2p+2)C_{mul}nnz(A)l+C_{mul}mlk+C_{svd}nl^2.
\end{aligned}
\end{equation}

\section{Methodology}
\subsection{The Ideas for Acceleration}
Because many real data can be regarded as sparse matrix, accelerating the basic rPCA algorithm for sparse matrix is the focus. In Alg. 1, the matrix multiplication in Steps 2 and 7 occupy the majority of computing time if $\mathbf{A}$ is dense. However, this is not true for sparse matrix, and therefore optimizing other steps will bring substantial acceleration.

In existing work, some ideas have been proposed to accelerate the basic rPCA algorithm. In \cite{rsvdpack}, the idea of using eigendecomposition to compute SVD in Step 8 of Alg. 1 was proposed. It was pointed out that in the power iteration, orthonormalization after each matrix multiplication is not necessary. In \cite{alg971}, the power iteration was accelerated by replacing QR factorization with LU factorization, and the Gaussian matrix was replaced with the random matrix with uniform distribution. In \cite{Erichson_2017_ICCV}, the rPCA algorithm without power iteration was discussed for dense matrix in image or video processing problem. It employs a variant of the basic rPCA algorithm, where the random matrix is multiplied to the left of $\mathbf{A}$. The algorithm was accelerated by using sparse random matrices and using eigendecomposition to obtain the orthonormal basis of the subspace.

For handling sparse $\mathbf{A}$, we just use the Gaussian matrix for $\mathbf{\Omega}$, because other matrices may cause $\mathbf{A\Omega}$ rank-deficient. The useful ideas for faster randomized PCA for sparse matrix are:
\begin{itemize}
\item Use the eigendecomposition for computing economic SVD of $\mathbf{B}$,
\item Replace the orthonormal $\mathbf{Q}$ with the left singular vector matrix $\mathbf{U}$,
\item Perform $\mathbf{LU}$ factorization in the power iteration,
\item Perform orthonormalization after every other matrix-matrix multiplication in power iteration.
\end{itemize}
Firstly, we formulate the eigendecomposition based SVD as an eigSVD algorithm (described in Alg. 2), where ``eig($\cdot$)'' computes the eigendecomposition and ``spdiags($\cdot$)'' is used for constructing a sparse diagonal matrix. In Alg. 2, the ``diag($\cdot$)'' in Step 3 is the function to transform a diagonal matrix to a vector. Step 4 is to construct a sparse diagonal matrix $\mathbf{\hat{S}}=\mathrm{diag}(\mathbf{S})^{-1}$, where ``.$/$'' is the element-wise division operator. The eigSVD algorithm's correctness is given as Lemma 1.

\begin{algorithm}
    \caption{eigSVD}
    \label{alg2}
    \begin{algorithmic}[1]
      \REQUIRE $\mathbf{A}\in\mathbb{R}^{m\times n}$ ($m\ge n$)
      \ENSURE $\mathbf{U}\in\mathbb{R}^{m\times n}$, $\mathbf{S}\in\mathbb{R}^n$, $\mathbf{V}\in\mathbb{R}^{n\times n}$
      \STATE $\mathbf{B} = \mathbf{A}^{\mathrm{T}}\mathbf{A}$
      \STATE $[\mathbf{V}, \mathbf{D}] = \mathrm{eig}(\mathbf{B})$
      \STATE $\mathbf{S} = \mathrm{sqrt}(\mathrm{diag}(\mathbf{D}))$
      \STATE $\mathbf{\hat{S}} = \mathrm{spdiags}(1./\mathbf{S},0,n,n)$
      \STATE $\mathbf{U} = \mathbf{A}\mathbf{V}\mathbf{\hat{S}}$
    \end{algorithmic}
  \end{algorithm}

\begin{lemman}
The matrix $\mathbf{U}$, $\mathbf{S}$ and $\mathbf{V}$ produced by Alg. 2 form the economic SVD of matrix $\mathbf{A}$.
\end{lemman}

\begin{proof}
Suppose $\mathbf{A}$ has SVD as (1). Since $m\ge n$,
\begin{equation}
\mathbf{A} = \mathbf{U}(:, 1:n)\hat{\mathbf{\Sigma}}\mathbf{V}^{\mathrm{T}}
\end{equation}
where $\hat{\mathbf{\Sigma}}$, a square diagonal matrix, is the first $n$ rows of $\mathbf{\Sigma}$. Eq. (5) is the economic SVD of $\mathbf{A}$. Then Step 1 computes
\begin{equation}
\mathbf{B} = \mathbf{A}^{\mathrm{T}}\mathbf{A}=\mathbf{V}\hat{\mathbf{\Sigma}}^2\mathbf{V}^{\mathrm{T}}.
\end{equation}
The right-hand side of (6) is the eigendecomposition of $\mathbf{B}$. This means in Step 2, $\mathbf{D}=\hat{\mathbf{\Sigma}}^2$ and $\mathbf{V}$ is the right singular vector matrix of $\mathbf{A}$. Therefore, the values of $\mathbf{S}$ in Step 3 are the diagonal elements of $\hat{\mathbf{\Sigma}}$ and the $\mathbf{\hat{S}}$ in Step 4 equals to $\hat{\mathbf{\Sigma}}^{-1}$. In Step 5, $\mathbf{U} = \mathbf{A}\mathbf{V}\mathbf{\hat{S}} = \mathbf{A}\mathbf{V}\hat{\mathbf{\Sigma}}^{-1} = \mathbf{U}(:, 1:n)$. The last equality is derived from (5). This proves the lemma.
\end{proof}

According to Alg. 2, the \emph{flop} count of the eigSVD algorithm is:
\begin{equation}
\begin{aligned}
\mathrm{FC}_{2}=2C_{mul}mn^2+C_{eig}n^3.\\
\end{aligned}
\end{equation}
We assume performing the eigendecomposition  of  an $n\times n$ matrix costs $C_{eig}n^3$ \emph{flops}.
Notice that eigSVD algorithm is especially efficient if $m\gg n$, because $\mathbf{B}$ becomes a small $n\times n$ matrix. And, the computed singular values in $\mathbf{S}$ are in ascending order. Numerical issues can arise if matrix $\mathbf{A}$ dose not have full column rank. So, eigSVD algorithm is only applicable to special situations.

Notice that ``eig($\cdot$)'' in Step 2 of Alg. 2 can be replaced with ``eigs($\cdot$)'' to compute the largest $k$ eigenvalues/eigenvectors, so that the algorithm can produce the results of truncated SVD. This results in an eigSVDs algorithm (see Algorithm 3), which can also be used to compute PCA.
\begin{algorithm}
    \caption{eigSVDs}
    \label{alg3}
    \begin{algorithmic}[1]
      \REQUIRE $\mathbf{A}\in\mathbb{R}^{m\times n}$ ($m\ge n$), $k$
      \ENSURE $\mathbf{U}\in\mathbb{R}^{m\times k}$, $\mathbf{S}\in\mathbb{R}^k$, $\mathbf{V}\in\mathbb{R}^{k\times n}$
      \STATE $\mathbf{B} = \mathbf{A}^{\mathrm{T}}\mathbf{A}$
      \STATE $[\mathbf{V}, \mathbf{D}] = \mathrm{eigs}(\mathbf{B}, k)$
      \STATE $\mathbf{S} = \mathrm{sqrt}(\mathrm{diag}(\mathbf{D}))$
      \STATE $\mathbf{\hat{S}} = \mathrm{spdiags}(1./\mathbf{S},0,k,k)$
      \STATE $\mathbf{U} = \mathbf{A}\mathbf{V}\mathbf{\hat{S}}$
    \end{algorithmic}
  \end{algorithm}

Secondly, the idea that the orthonormal $\mathbf{Q}$ can be replaced with the left singular matrix $\mathbf{U}$ can be explained with Lemma 2.

\begin{lemman}
In the basic rPCA algorithm, orthonormal matrix $\mathbf{Q}$ includes a set of orthonormal basis of subspace $range(\mathbf{A\Omega})$ or $range((\mathbf{AA}^{\mathrm{T}})^p\mathbf{A\Omega})$. As long as $\mathbf{Q}$ holds this property, no matter how it is produced, the result of basic rPCA algorithm will not change. 
\end{lemman}
\begin{proof}
From Step 2 of Alg. 1 we see that $\mathbf{Q}$ is an orthonormal matrix, and its columns are a set of orthonormal basis of subspace $range(\mathbf{A\Omega})$. If $p>0$, from Steps 3$\sim$6 we can see that the orthonormal matrix $\mathbf{Q}$ includes the a set of orthonormal basis of subspace $range((\mathbf{AA}^{\mathrm{T}})^p\mathbf{A\Omega})$. The result of the basic rPCA algorithm, is actually $\mathbf{QB} = \mathbf{QQ}^{\mathrm{T}}\mathbf{A}$, which further equals $\mathbf{US V}^{\mathrm{T}}$. Notice that $\mathbf{QQ}^{\mathrm{T}}$ is an orthogonal projector onto the subspace $range(\mathbf{Q})$, if $\mathbf{Q}$ is an orthonormal matrix. The orthogonal projector is uniquely determined by the subspace (see \citep{matrix1996} or Section 8.2 of \citep{Halko2011Finding}), i.e. $range(\mathbf{A\Omega})$ or $range((\mathbf{AA}^{\mathrm{T}})^p\mathbf{A\Omega})$. Therefore, As long as $\mathbf{Q}$ includes a set of orthonormal basis of the subspace, $ \mathbf{QQ}^{\mathrm{T}}$ is identical and the basic rPCA algorithm's results will not change.
\end{proof}

Both QR factorization and SVD of a same matrix produce the orthonormal basis of its range space (column space), in $\mathbf{Q}$ and $\mathbf{U}$, respectively. Therefore, with Lemma 2, we see that $\mathbf{Q}$ can be replaced by $\mathbf{U}$ from SVD in the basic rPCA algorithm.

Thirdly, LU factorization is used in power iteration to replace QR factorization for saving runtime. This does not affecting the algorithm's correctness, which is proved in Lemma 3.
\begin{lemman}
In the basic rPCA algorithm, the ``orth($\cdot$)'' operation in the power iteration, except the last one, can be replaced by LU factorization. This does not affect the algorithm's accuracy in exact arithmetic.
\end{lemman}
\begin{proof}
Firstly, if the ``orth($\cdot$)'' is not performed, the power iteration produces $\mathbf{Q}$ including a set of basis of the subspace $range((\mathbf{AA}^{\mathrm{T}})^p\mathbf{A\Omega})$. As mentioned before, the ``orth($\cdot$)'' is just for alleviating the round-off error, and after using it $\mathbf{Q}$ still represents $range((\mathbf{AA}^{\mathrm{T}})^p\mathbf{A\Omega})$.

The pivot LU factorization of a matrix $\mathbf{M}$ is:
\begin{equation}
\mathbf{PM} = \mathbf{LU},
\end{equation}
where $\mathbf{P}$ is a permutation matrix, and $\mathbf{L}$ and $\mathbf{U}$ are lower triangular and upper triangular matrices, respectively. Obviously, $\mathbf{M} = (\mathbf{P}^{\mathrm{T}}\mathbf{L})\mathbf{U}$, where $\mathbf{P}^{\mathrm{T}}\mathbf{L}$ has the same column space as $\mathbf{M}$. Therefore, replacing ``orth($\cdot$)'' with LU factorization (using $\mathbf{P}^{\mathrm{T}}\mathbf{L}$) also produces the basis of $range((\mathbf{AA}^{\mathrm{T}})^p\mathbf{A\Omega})$. Then, based on Lemma 2, this does not affect the algorithm's result in exact arithmetic.
\end{proof}

Notice that the LU factor $\mathbf{P}^{\mathrm{T}}\mathbf{L}$ has scaled matrix entries with linearly independent columns, since $\mathbf{L}$ is a lower triangular matrix with unit diagonals and $\mathbf{P}$ just means row permutation. Therefore, it also alleviates the round-off error. 

Finally, the orthonormalization or LU factorization in power iteration can be performed after every other matrix multiplication. It harms the accuracy little but remarkably reduces runtime.

\subsection{A Modified Power Iteration Scheme and Handling Matrix with More Columns}
For a sparse matrix $\mathbf{A}$, the power iteration in the basic rPCA algorithm (Alg. 1) is computationally expensive, because it includes the multiplication of two dense matrices. We also notice that, each time we increase the power parameter by one, two matrix multiplications are induced resulting in large increase of computation cost. This makes inconvenient trade-off between runtime and accuracy. To alleviate this issue, we here propose a modified power iteration scheme, which allows odd number of passes over $\mathbf{A}$ and thus provides more convenient performance trade-off of the rPCA algorithm.

We first observe that,  if the power parameter $p>0$, Steps 1 and 2 of Alg. 1 can be simply replaced with:
\begin{algorithm}
    \begin{algorithmic}[1]
      \STATE $\mathbf{Q} = \mathbf{\Omega} = \mathrm{randn}(m, k+s)$
    \end{algorithmic}
  \end{algorithm}
\\For the same power parameter $p$, this reduces one pass over matrix $\mathbf{A}$. Because the singular values of $(\mathbf{AA}^{\mathrm{T}})^p$ decay more quickly than $(\mathbf{AA}^{\mathrm{T}})^{p-1}\mathbf{A}$, performing randomized QB procedure on $(\mathbf{AA}^{\mathrm{T}})^p$ is more accurate than $(\mathbf{AA}^{\mathrm{T}})^{p-1}\mathbf{A}$. It means:
\begin{equation}
||\mathbf{A}-\mathbf{\hat{Q}\hat{Q}}^\mathrm{T}\mathbf{A}||<||\mathbf{A}-\mathbf{QQ}^\mathrm{T}\mathbf{A}||,
\end{equation}
where $\mathbf{\hat{Q}}$ is the orthogonal matrix produced by $(\mathbf{AA}^{\mathrm{T}})^p\mathbf{\Omega}$ while $\mathbf{Q}$ is produced from $(\mathbf{AA}^{\mathrm{T}})^{p-1}\mathbf{A\Omega}$. This proves the rationality of this modification. Therefore, we can just modify Step 1 and 2 in Alg. 1 like this without other modification, to realize the odd number of passes over $\mathbf{A}$.

Another modification of Alg. 1 can be motivated by the observation that the \emph{flop} count of Alg. 1, i.e. (4),  is not favorable to the case with $m<n$, because $C_{svd}$ is much large than other constants ($C_{qr}$ and $C_{mul}$). Although we may run the algorithm to process $\mathbf{A}^{\mathrm{T}}$, the transpose of a sparse matrix is not easily obtained due to the storage format of sparse matrix.

Actually, there is a variant of the basic rPCA algorithm \citep{Erichson_2017_ICCV,alg971}, where the random matrix is multiplied to the left of $\mathbf{A}$. With the same idea, we derive an algorithm called basic rPCAt described as Alg. 4. Its flop count is:
\begin{equation}
\begin{aligned}
\mathrm{FC}_{4}=pC_{qr}ml^{2}+(p+1)C_{qr}nl^2+(2p+2)C_{mul}nnz(A)l+C_{mul}nlk+C_{svd}ml^2.
\end{aligned}
\end{equation}
Therefore, we can derive that when Alg. 1 and Alg. 4 handle a sparse matrix $\mathbf{A}\in\mathbb{R}^{m\times n}$ with $m<n$,
\begin{equation}
\mathrm{FC}_{1} - \mathrm{FC}_{4}=
(C_{svd}-C_{qr}-C_{mul})(n-m)l^2>0.
\end{equation}
The reason is that $C_{svd}$ is much larger than $C_{qr}$ and $C_{mul}$.
Eq. (11) shows Alg. 4 is more efficient than Alg. 1 when handling the matrix with more columns. Thus, we shall choose between them according to the matrix's dimensions, so as to achieve the best runtime performance.

\begin{algorithm}
    \caption{basic rPCAt}
    \label{alg4}
    \begin{algorithmic}[1]
      \REQUIRE $\mathbf{A}\in\mathbb{R}^{m\times n}$, rank parameter $k$, power parameter $p$
      \ENSURE $\mathbf{U}\in\mathbb{R}^{m\times k}$, $\mathbf{S}\in\mathbb{R}^{k\times k}$, $\mathbf{V}\in\mathbb{R}^{n\times k}$
      \STATE $\mathbf{\Omega} = \mathrm{randn}(k+s, m)$
      \STATE $\mathbf{Q} = \mathrm{orth}((\mathbf{\Omega}\mathbf{A})^{\mathrm{T}})$
      \FOR {$i=1, 2, \cdots, p$}
        \STATE $\mathbf{G} = \mathrm{orth}(\mathbf{A}\mathbf{Q})$
        \STATE $\mathbf{Q} = \mathrm{orth}(\mathbf{A}^{\mathrm{T}}\mathbf{G})$
      \ENDFOR
      \STATE $\mathbf{B}  =(\mathbf{A}\mathbf{Q})^{\mathrm{T}}$
      \STATE $[\mathbf{\hat{U}}, \mathbf{\hat{S}}, \mathbf{\hat{V}}] = \mathrm{svd}(\mathbf{B})$
      \STATE $\mathbf{V} = \mathbf{Q}\mathbf{\hat{U}}$
      \STATE $\mathbf{U} = \mathbf{\hat{V}}(:, 1:k), \mathbf{S} = \mathbf{\hat{S}}(1:k, 1:k), \mathbf{V} = \mathbf{V}(:, 1:k)$
    \end{algorithmic}
  \end{algorithm}

\subsection{The Fast PCA Algorithm for Sparse Data}
Based on Section 3.1 , we find out that the eigSVD procedure can be applied to the basic rPCA algorithm to produce both the economic SVD of $\mathbf{B}$ and the orthonormal $\mathbf{Q}$. Because $k+s\ll \min(m, n)$ and matrices are not rank-deficient in practice, using eigSVD induces no numerical issue. With the accelerating skills, we propose a fast rPCA algorithm (\emph{frPCA}) for sparse matrix (described as Alg. 5), where ``lu($\cdot$)'' denotes the LU factorization and its first output is ``$\mathbf{P}^{\mathrm{T}}\mathbf{L}$''. ``$\lfloor\cdot\rfloor$'' denotes the floor function which returns the maximum integer no larger than the input value, and the pass parameter $q$ represents the number of passes over $\mathbf{A}$ in the whole algorithm. The equivalence between the frPCA algorithm and the basic rPCA algorithm is demonstrated as Theorem 1.

\begin{algorithm}
    \caption{frPCA}
    \label{alg5}
    \begin{algorithmic}[1]
      \REQUIRE $\mathbf{A}\in\mathbb{R}^{m\times n}$~$(m \le n)$, $k$, pass parameter $q\ge 2$
      \ENSURE $\mathbf{U}\in\mathbb{R}^{m\times k}$, $\mathbf{S}\in\mathbb{R}^{k}$, $\mathbf{V}\in\mathbb{R}^{n\times k}$
     \IF{$q$ is an even number}
	\STATE $\mathbf{Q} = \mathrm{randn}(n, k+s)$
	\STATE $\mathbf{Q}=\mathbf{A}\mathbf{Q}$
	\STATE \textbf{if} $q>2$ \textbf{then} $[\mathbf{Q},\sim]=\mathrm{lu}(\mathbf{Q})$ \textbf{else} $[\mathbf{Q},\sim,\sim]=\mathrm{eigSVD}(\mathbf{Q})$
	\ELSE
	\STATE $\mathbf{Q} = \mathrm{randn}(m, k+s)$
	\ENDIF
      \FOR {$i=1, 2, 3, \cdots, \lfloor\frac{q-1}{2}\rfloor$}
        \IF {$i==\lfloor\frac{q-1}{2}\rfloor$}
		\STATE $[\mathbf{Q},\sim,\sim] = \mathrm{eigSVD}(\mathbf{A}(\mathbf{A}^\mathrm{T}\mathbf{Q}))$
	\ELSE
        	\STATE $[\mathbf{Q},\sim] = \mathrm{lu}(\mathbf{A}(\mathbf{A}^\mathrm{T}\mathbf{Q}))$
      \ENDIF
	\ENDFOR
	\STATE $[\mathbf{\hat{U}}, \mathbf{\hat{S}}, \mathbf{\hat{V}}] = \mathrm{eigSVD}(\mathbf{A}^\mathrm{T}\mathbf{Q})$
	 \STATE $ind = k+s:-1:s+1$
	\STATE $\mathbf{U} = \mathbf{Q}\mathbf{\hat{V}}(:,ind), \mathbf{V} = \mathbf{\hat{U}}(:, ind), \mathbf{S} = \mathbf{\hat{S}}(ind)$
    \end{algorithmic}
\end{algorithm}

\begin{theoremn}
The frPCA algorithm (Alg. 5) is mathematically equivalent to the basic rPCA algorithm (Alg. 1) when $p=(q-2)/2$.
\end{theoremn}
\begin{proof}
When $p=(q-2)/2$, the number of power iteration is the same for the both algorithms. One difference between Alg. 5 and Alg. 1 is in the power iteration (the "for" loop). Based on Lemma 1 we see that eigSVD accurately produces a set of orthonormal basis. Besides, based on Lemma 2 and 3, we see the power iteration in Alg. 5 is mathematically equivalent to that in Alg. 1. The other difference is the last three steps in Alg. 5. Its correctness is due to Lemma 1 and that the singular values produces by eigSVD is in the ascending order.
\end{proof}

Below we analyze the flop count of Alg. 5.
Suppose the flop count of multiplication of $\mathbf{M}\in\mathbb{R}^{m\times l}$ and $\mathbf{N}\in\mathbb{R}^{l\times l}$ is $C_{mul}ml^2$. Here, $C_{mul}$ reflects one addition and one multiplication. If LU factorization is performed on $\mathbf{M}$, it takes $ml^2-l^3/2$ times minus and multiplication operations, denoted by $C_{lu}(ml^2-l^3/2)$ flop counts. If $l \ll m$, we see that LU factorization costs similar runtime as the the matrix multiplication. So, for the purpose of runtime comparison, we assume that  $C_{mul}ml^2\approx C_{lu}(ml^2-l^3/2)$ in the following analysis. Considering that $l \ll \min(m,n)$, we derive the flop count of Alg. 5 for the situation with $q$ equal an even number:

\begin{equation}
\begin{aligned}
\mathrm{FC}_{5}&= (\frac{q}{2}-1)C_{lu}(ml^2-\frac{l^3}{2})+qC_{mul}nnz(A)l+C_{mul}mlk+2C_{mul}(m+n)l^2+2C_{eig}l^3\\
&\approx(\frac{q}{2}-1)C_{mul}ml^2+qC_{mul}nnz(A)l+C_{mul}mlk+2C_{mul}(m+n)l^2.
\end{aligned}
\end{equation}


As we will see soon, Alg. 5 is more efficient for handling matrix $\mathbf{A}$ with dimension $m<n$. So, we also propose a variant fast rPCA algorithm (denoted by \emph{frPCAt}) through applying the accelerating skills to Alg. 4. The resulted algorithm is described as Alg. 6. 
\begin{algorithm}
    \caption{frPCAt}
    \label{alg6}
    \begin{algorithmic}[1]
      \REQUIRE $\mathbf{A}\in\mathbb{R}^{m\times n}$ ($m\ge n$), $k$, pass parameter $q\ge 2$
      \ENSURE $\mathbf{U}\in\mathbb{R}^{m\times k}$, $\mathbf{S}\in\mathbb{R}^{k}$, $\mathbf{V}\in\mathbb{R}^{n\times k}$
     \IF{$q$ is an even number}
	\STATE $\mathbf{Q} = \mathrm{randn}(k+s, m)$
	\STATE $\mathbf{Q} =(\mathbf{Q}\mathbf{A})^\mathrm{T}$
	\STATE \textbf{if} $q==2$ \textbf{then} $[\mathbf{Q},\sim,\sim]=\mathrm{eigSVD}(\mathbf{Q})$ \textbf{else} $[\mathbf{Q},\sim]=\mathrm{lu}(\mathbf{Q})$
	\ELSE
	\STATE $\mathbf{Q} = \mathrm{randn}(n, k+s)$
	\ENDIF
      \FOR {$i=1, 2, 3, \cdots, \lfloor\frac{q-1}{2}\rfloor$}
        \IF {$i==\lfloor\frac{q-1}{2}\rfloor$}
		\STATE $[\mathbf{Q},\sim,\sim] = \mathrm{eigSVD}(\mathbf{A}^\mathrm{T}(\mathbf{A}\mathbf{Q}))$
	\ELSE
        	\STATE $[\mathbf{Q},\sim] = \mathrm{lu}(\mathbf{A}^\mathrm{T}(\mathbf{A}\mathbf{Q}))$
      \ENDIF
	\ENDFOR
	\STATE $[\mathbf{\hat{U}}, \mathbf{\hat{S}}, \mathbf{\hat{V}}] = \mathrm{eigSVD}(\mathbf{AQ})$
	 \STATE $ind = k+s:-1:s+1$
	\STATE $\mathbf{U} = \mathbf{\hat{U}}(:, ind), \mathbf{V} = \mathbf{Q}\mathbf{\hat{V}}(:,ind), \mathbf{S} = \mathbf{\hat{S}}(ind)$
    \end{algorithmic}
\end{algorithm}

\begin{theoremn}
The variant fast rPCA algorithm (Alg. 6) is mathematically equivalent to the basic rPCAt algorithm (Alg. 4) when $p=(q-2)/2$.
\end{theoremn}
\begin{proof}
When $p=(q-2)/2$, the number of power iteration is the same. The differences between Alg. 6 and Alg. 4 are also in the power iteration and the last three steps in Alg. 6. Based on Lemma 1 we see that eigSVD accurately produces a set of orthonormal basis. Besides, based on Lemma 2 and 3, we see the power iteration in Alg. 6 is mathematically equivalent to that in Alg. 4.  The correctness of last three steps is due to Lemma 1 and that the singular values produces by eigSVD is in the ascending order.
\end{proof}

Similarly, we can analyze the \emph{flop} count of the variant fast rPCA algorithm (Alg. 6).
\begin{equation}
\begin{aligned}
\mathrm{FC}_{6}&= (\frac{q}{2}-1)C_{lu}(nl^2-\frac{l^3}{2})+qC_{mul}nnz(A)l+C_{mul}nlk+2C_{mul}(m+n)l^2+2C_{eig}l^3\\
&\approx(\frac{q}{2}-1)C_{mul}nl^2+qC_{mul}nnz(A)l+C_{mul}nlk+2C_{mul}(m+n)l^2.
\end{aligned}
\end{equation}
Now, the difference of flop counts of Alg. 5 and 6 is  
\begin{equation}
\mathrm{FC}_{6} - \mathrm{FC}_{5}=
(\frac{q}{2}-1)C_{lu}(n-m)l^2+C_{mul}(n-m)lk<0,
\end{equation}
if they are performed on $\mathbf{A}\in\mathbb{R}^{m\times n}~(m>n)$.
It means Alg. 6 is more efficient for handling matrix with more rows. 
Accordingly, Alg. 5 is more efficient for handling matrix with more columns.

To evaluate how the proposed fast PCA algorithm accelerates the basic rPCA algorithm, we give the following analysis on the theoretical speedup based on flop counts. As Alg. 1 and Alg. 6 are both efficient for the situation with $m \ge n$, we analyze the ratio of flop counts of them. According to (4) and (13), the speedup ratio of the \emph{frPCAt} algorithm to the basic rPCA algorithm is (assuming that $p=(q-2)/2$):
\begin{equation}
\begin{aligned}
\mathrm{Sp1} &= \frac{\mathrm{FC}_{1}}{\mathrm{FC}_{6}} \\
&\approx \frac{(\frac{q}{2}-1)C_{qr}nl^{2}+\frac{q}{2}C_{qr}ml^2+qC_{mul}nnz(A)l+C_{mul}mlk+C_{svd}nl^2}{{(\frac{q}{2}-1)C_{mul}nl^2+qC_{mul}nnz(A)l+C_{mul}nlk+2C_{mul}(m+n)l^2}}.
\end{aligned}
\end{equation}
Denote $t=nnz(\mathbf{A})/m $ as the average number of nonzeros per row, $\alpha=t/l$ as a sparsity parameter related to the rank parameter, 
and $\beta=n/m$ as a matrix shape parameter ($\beta \le 1$). We further derive:
\begin{equation}
\begin{aligned}
\mathrm{Sp1}\approx \frac{(\frac{q}{2}-1)C_{qr}\beta+\frac{q}{2}C_{qr}+qC_{mul}\alpha+C_{mul}+C_{svd}\beta}{(\frac{q}{2}+1)C_{mul}\beta+qC_{mul}\alpha+C_{mul}\beta+2C_{mul}}.
\end{aligned}
\end{equation}
Based on this, we have the following theorem.

\begin{theoremn}
The speedup ratio of the variant fast PCA algorithm (Alg. 6) to the basic rPCA algorithm (Alg. 1),  Sp1, depends on the number of passes over $\mathbf{A}$ (denoted by $q$), the ratio of average number of nonzeros per row to the rank parameter $l$ (denoted by $\alpha$), and the number of columns over the number of rows (denoted by $\beta$). Sp1 becomes higher as $\alpha$ decreases. And,
\begin{equation}
\begin{aligned}
\lim_{q\to \infty}{\mathrm{Sp1}}= \frac{C_{qr}\beta+C_{qr}+2C_{mul}\alpha}{C_{mul}\beta+2C_{mul}\alpha},
\end{aligned}
\end{equation}
which approaches to $2C_{qr}/C_{mul}$ for a very sparse square matrix $\mathbf{A}$ ($\alpha$ is small and $\beta$ equals 1). Here, $C_{qr}$ and $C_{mul}$ are the constants for the flop counts of QR factorization and matrix-matrix multiplication respectively. 
\end{theoremn}
\begin{proof}
Firstly, based on (15) and (16), we can derive the derivative of Sp1 with respect to $\alpha$:
\begin{equation}
\frac{\partial\mathrm{Sp1}}{\partial\alpha}=\frac{qC_{mul}(\mathrm{FC_6}-\mathrm{FC_1})}{\mathrm{FC_6}^2}<0,
\end{equation}
where $\mathrm{FC}_1$ and $\mathrm{FC}_6$ represents the \emph{flop} counts of Alg. 1 and 6, respectively. This means Sp1 increases with the decrease of $\alpha$ (which means matrix becomes sparser). 

Then, when $q$ is sufficiently large, we can derive:
\begin{equation}
\mathrm{Sp1}\approx\frac{(\frac{q}{2}-1)C_{qr}\beta+\frac{q}{2}C_{qr}+qC_{mul}\alpha}{(\frac{q}{2}+1)C_{mul}\beta+qC_{mul}\alpha} \approx \frac{C_{qr}\beta+C_{qr}+2C_{mul}\alpha}{C_{mul}\beta+2C_{mul}\alpha},
\end{equation}
which results in (17). Finally, if $\alpha \to 0$ and $\beta=1$ the speedup ratio approaches to $2C_{qr}/C_{mul}$. This is the upper bound of the speedup for a square or approximately square $\mathbf{A}$, absolutely greater than 1.
\end{proof}

A similar theorem can be derived for Alg. 5. With the theorems, we see that the proposed fast rPCA algorithm accelerates the basic rPCA algorithm without loss of accuracy. Besides, it allows odd number of passes over matrix $\mathbf{A}$, providing better trade-off between runtime and accuracy. 

\section{Experiments}

All experiments are carried out on a Linux server with two 12-core Intel Xeon E5-2630 CPUs (2.30 GHz), and 32 GB RAM. The proposed algorithms Alg. 5 and Alg. 6 are implemented in C with MKL libraries \citep{Intel} and OpenMP derivatives for multi-thread computing. QR factorization, LU factorization and other basic linear algebra operations are realized through LAPACK routines which are automatically executed in parallel on the multi-core CPUs. \texttt{svds} in Matlab2016b are used as accurate truncated SVD. eigSVDs is the other algorithm to compare and is efficiently implemented in Matlab2016b. Because the \texttt{lansvd} in Matlab/Fortran is not well parallelized, it runs slower than \texttt{svds} in our experiments.  And, considering that $k\ll \min (m, n)$, the method  calculating all the singular values/vectors by eigSVD and then making truncation is not competitive in runtime. Therefore, we do not include \texttt{lansvd} and eigSVD in the comparisons.

In the experiments, we choose Alg. 5 or 6 as the proposed fast algorithm according to the shape of test matrix. The oversampling parameter is always set $s=5$, and all runtimes are in seconds.



\subsection{Accuracy and Efficiency Validation}
In this subsection, we test different sparse matrices from real data and show how the performance of the proposed algorithm is affected by various factors.

Firstly, 
we obtain a sparse matrix in size 90,230 $\times$ 45,115 from the MovieLens dataset \citep{movielens}. The matrix has 97 nonzeros per row on average and is denoted by Matrix 1. Then, we randomly set some nonzero elements to zero to get two sparser matrices: Matrix 2 and 3 with 24 and 9 nonzeros per row on average, respectively. We also obtain three matrices (Matrix 4-6) from the information retrieval application ``AMiner'' \citep{Aminer}. They are in size 647,789 $\times$ 323,896, and have about 16, 8 and 4 nonzeros per row on average, respectively. We compute the first 100 singular values and singular vectors with \texttt{svds}, eigSVDs, Alg. 1 (setting $p=5$) and Alg. 6 (setting $q=11$). The runtimes are listed in Table 1 and the computed singular values are drawn in Figure 1. We use ``Sp2'' to denote the speedup ratio to \texttt{svds}, which is different from Sp1.

\begin{table}[h]
 \setlength{\abovecaptionskip}{0.05 cm}
 \caption{The runtimes of different PCA algorithms for matrices with different sparsity.}
  \label{tab:table2}
  \centering
\small{
\begin{spacing}{0.9}
\renewcommand{\multirowsetup}{\centering}
  \begin{tabular}{@{}c@{~~}c@{~~}c@{~~}c@{~~}c@{~~}c@{~~}c@{~~}c@{~~}c@{~~}c@{~~}c@{~~}c@{~~}c@{~~}c@{}} 
  \toprule
    \multirow{2}{*}{Algorithm}  &
  \multicolumn{2}{c}{Matrix 1} & \multicolumn{2}{c}{Matrix 2} &\multicolumn{2}{c}{Matrix 3} &\multicolumn{2}{c}{Matrix 4} &\multicolumn{2}{c}{Matrix 5} &\multicolumn{2}{c}{Matrix 6}\\
  \cmidrule(r){2-3} \cmidrule(r){4-5} \cmidrule(r){6-7} \cmidrule(r){8-9} \cmidrule(r){10-11} \cmidrule(r){12-13}
    & time & Sp2 & time & Sp2 & time & Sp2 & time & Sp2 & time & Sp2 & time & Sp2 \\
\midrule
 svds & 36.0 & * & 25.5 & * & 21.0 & * & 178.9 & * & 149.5 & * & 131.1 & *\\
 eigSVDs & 459.2 & 0.1 & 104.6 & 0.2 & 37.2 & 0.6 & 278.7 & 0.6 & 156.2 & 1.1 & 75.7 & 1.7\\
 Alg.1 ($p=5$) & 13.1 & 2.8  &   10.0 & 2.5  &  9.76  &  2.2  & 99.9 & 1.8 & 90.7  & 1.6 & 84.5 & 1.6 \\
 Alg.6 ($q=11$) & 4.32 & \textbf{8.3} & 1.58 &  \textbf{16}  &  1.05 & \textbf{20}  & 17.2 & \textbf{10}  & 13.8  & \textbf{11} & 10.2 & \textbf{13} \\
 \bottomrule 
 \end{tabular}
 \end{spacing}
 }
\end{table}

\begin{figure}[h]
  \centering
  \setlength{\abovecaptionskip}{0.02 cm}
  \subfigure[Matrix 1] {\includegraphics[width=1.8in, height=1.8in]{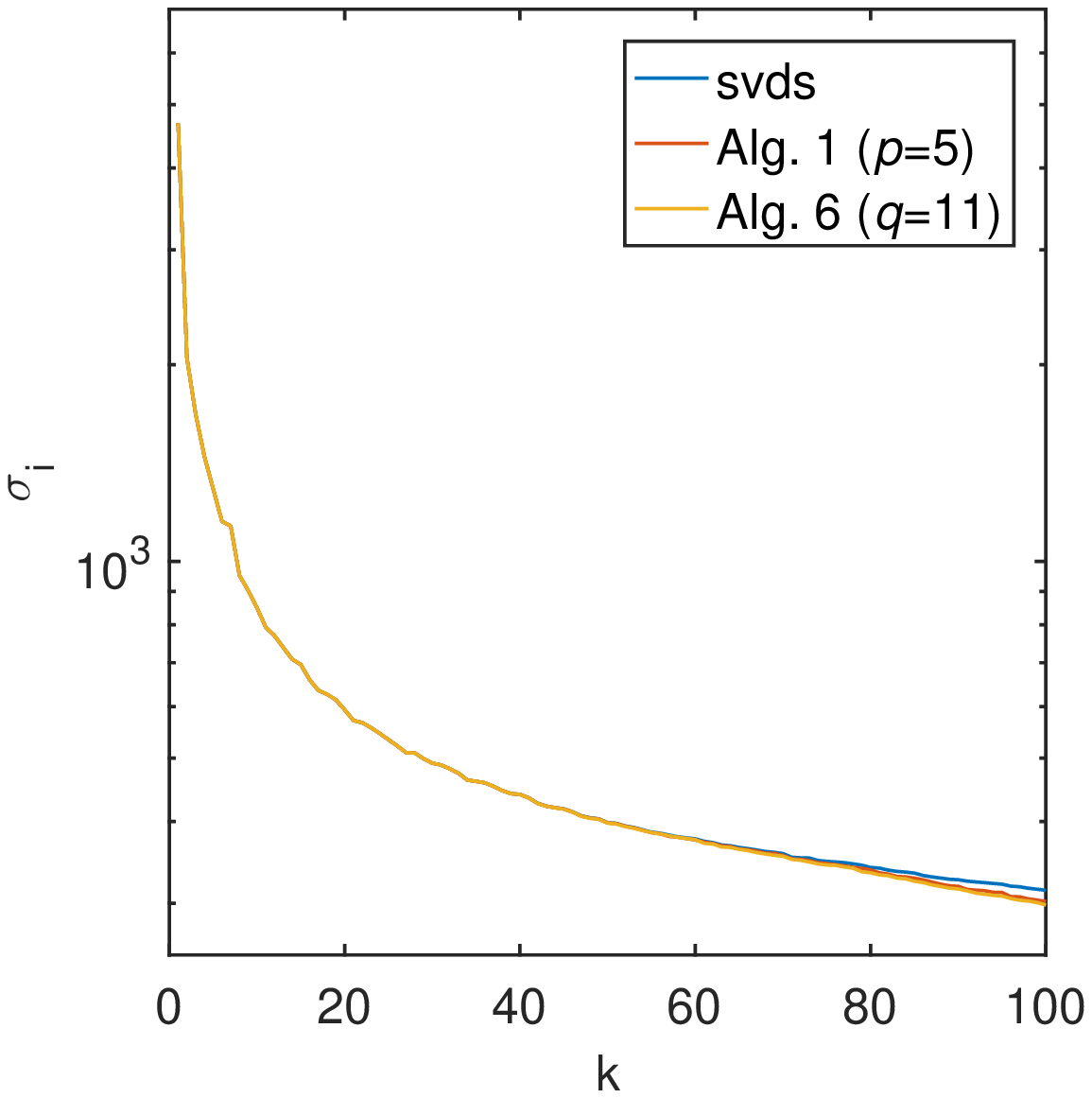}}
  \subfigure[Matrix 2] {\includegraphics[width=1.8in, height=1.8in]{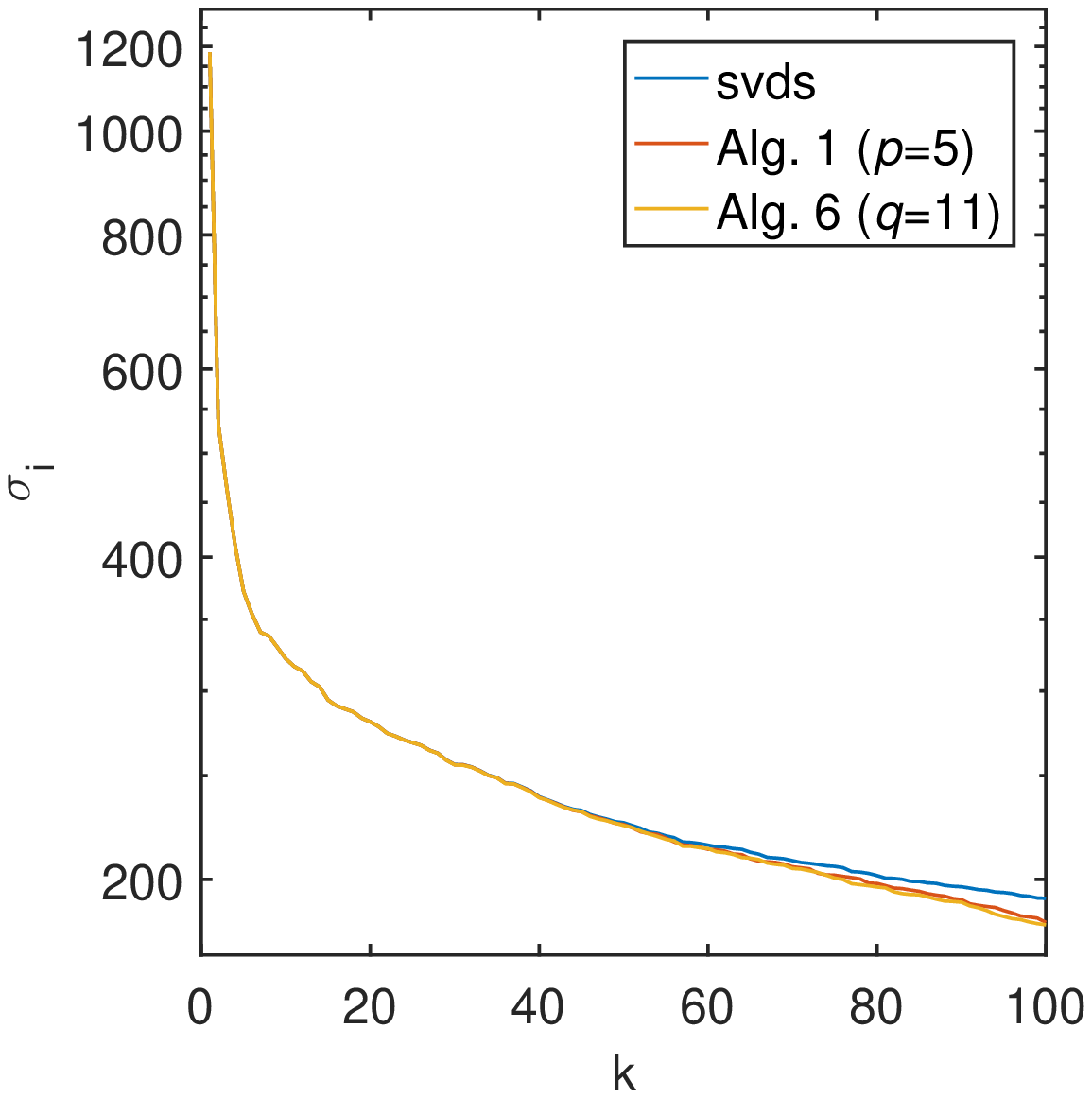}}
  \subfigure[Matrix 3] {\includegraphics[width=1.8in, height=1.8in]{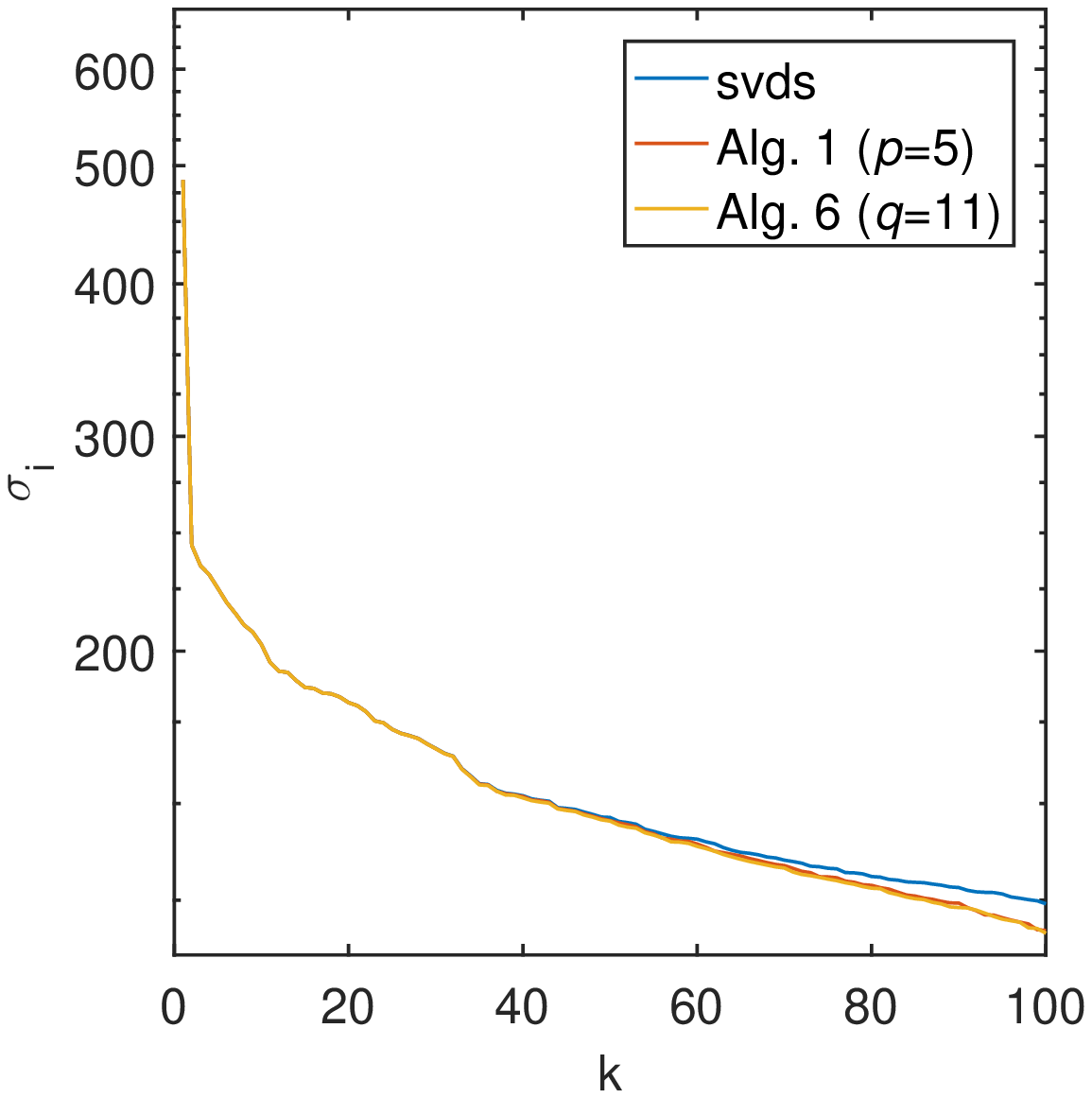}}
  \subfigure[Matrix 4] {\includegraphics[width=1.8in, height=1.8in]{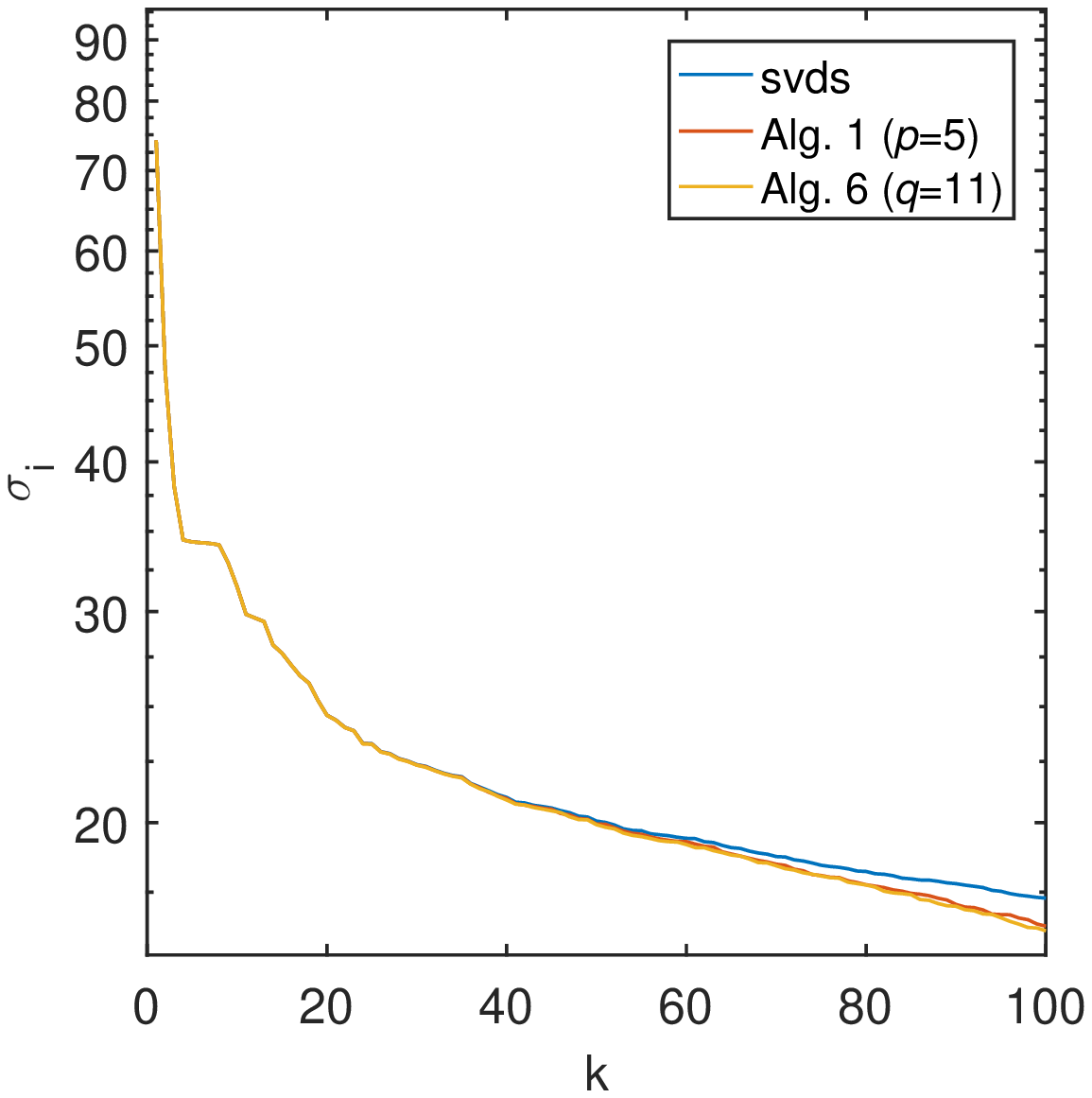}}
  \subfigure[Matrix 5] {\includegraphics[width=1.8in, height=1.8in]{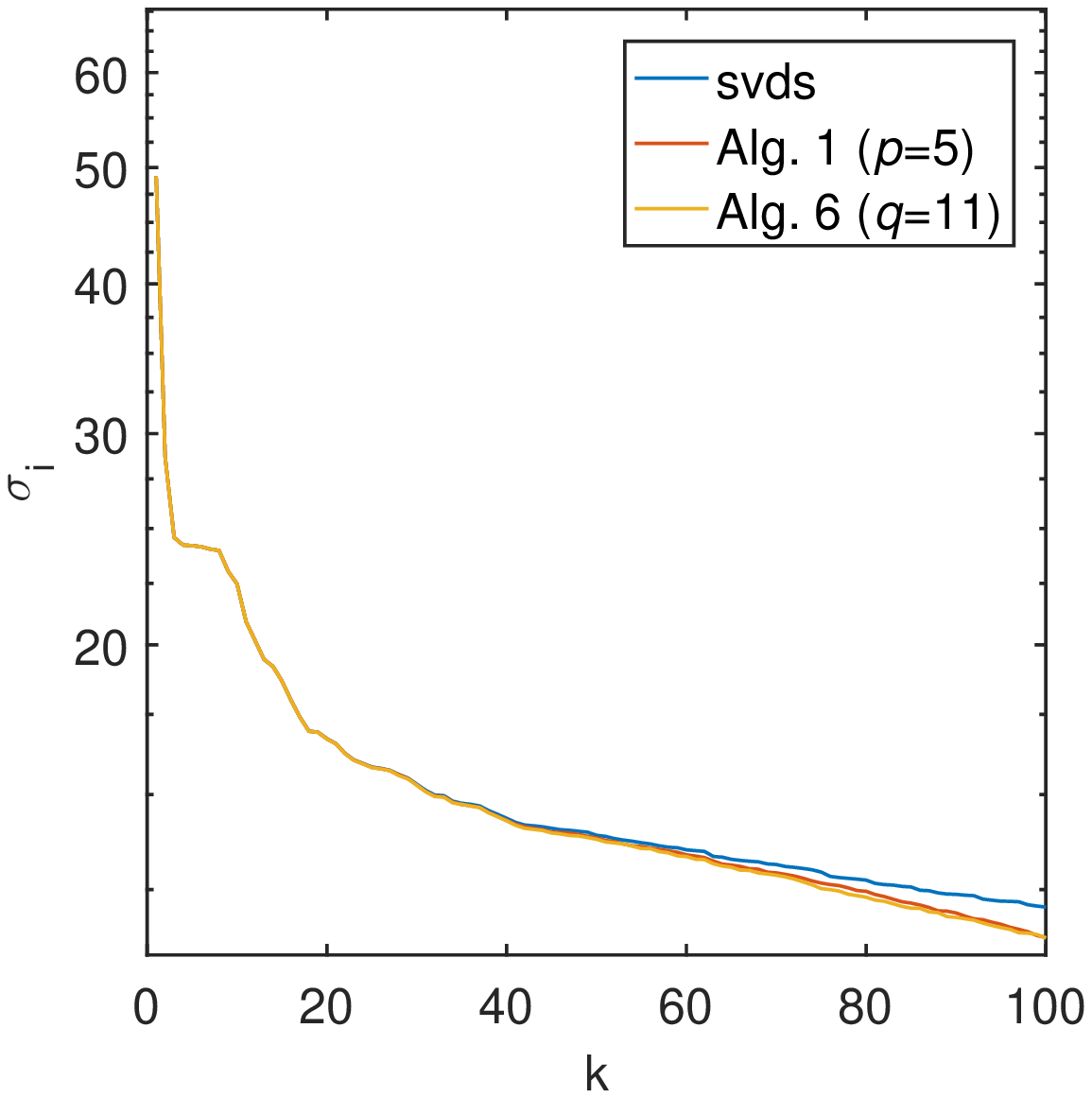}}
  \subfigure[Matrix 6] {\includegraphics[width=1.8in, height=1.8in]{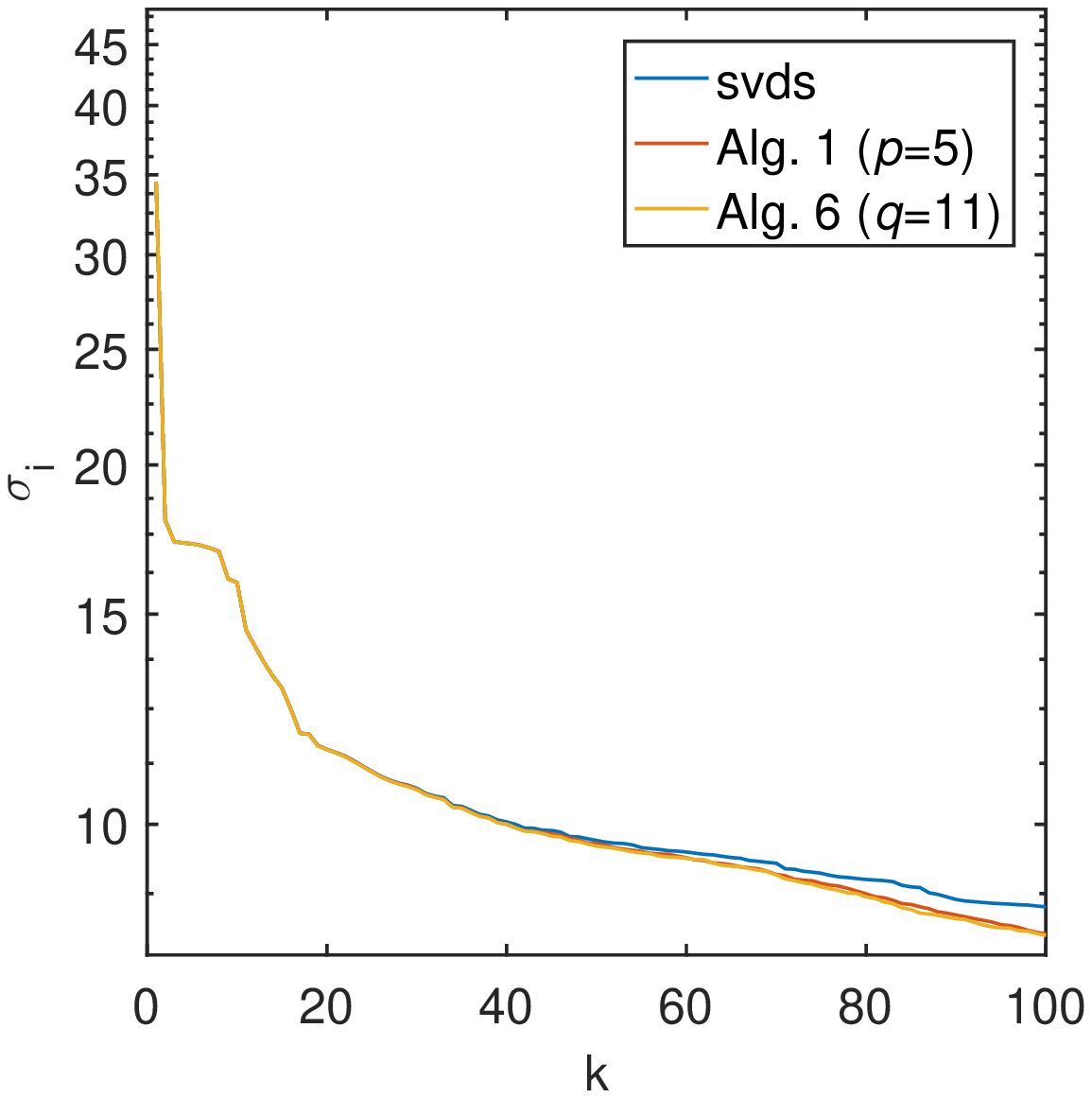}}
  \caption{The computed singular values for different matrices, showing the accuracy of our algorithm.}
  \label {Fig1}
\end{figure}

\begin{figure}[h]     
  \centering
  \setlength{\abovecaptionskip}{0.02 cm}
  \subfigure[] {\includegraphics[width=3.6in, height=1.8in]{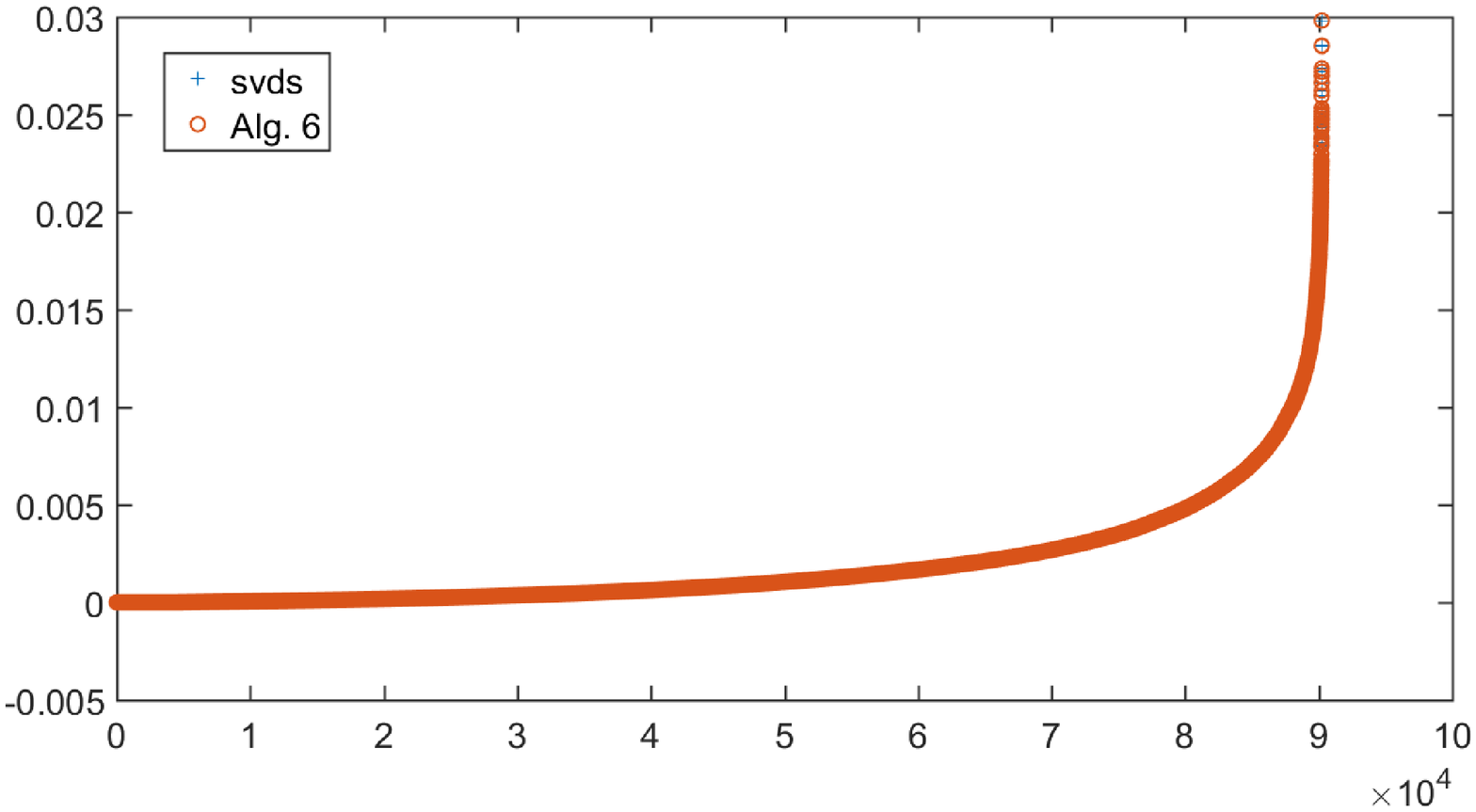}}
  \subfigure[] {\includegraphics[width=1.8in, height=1.8in]{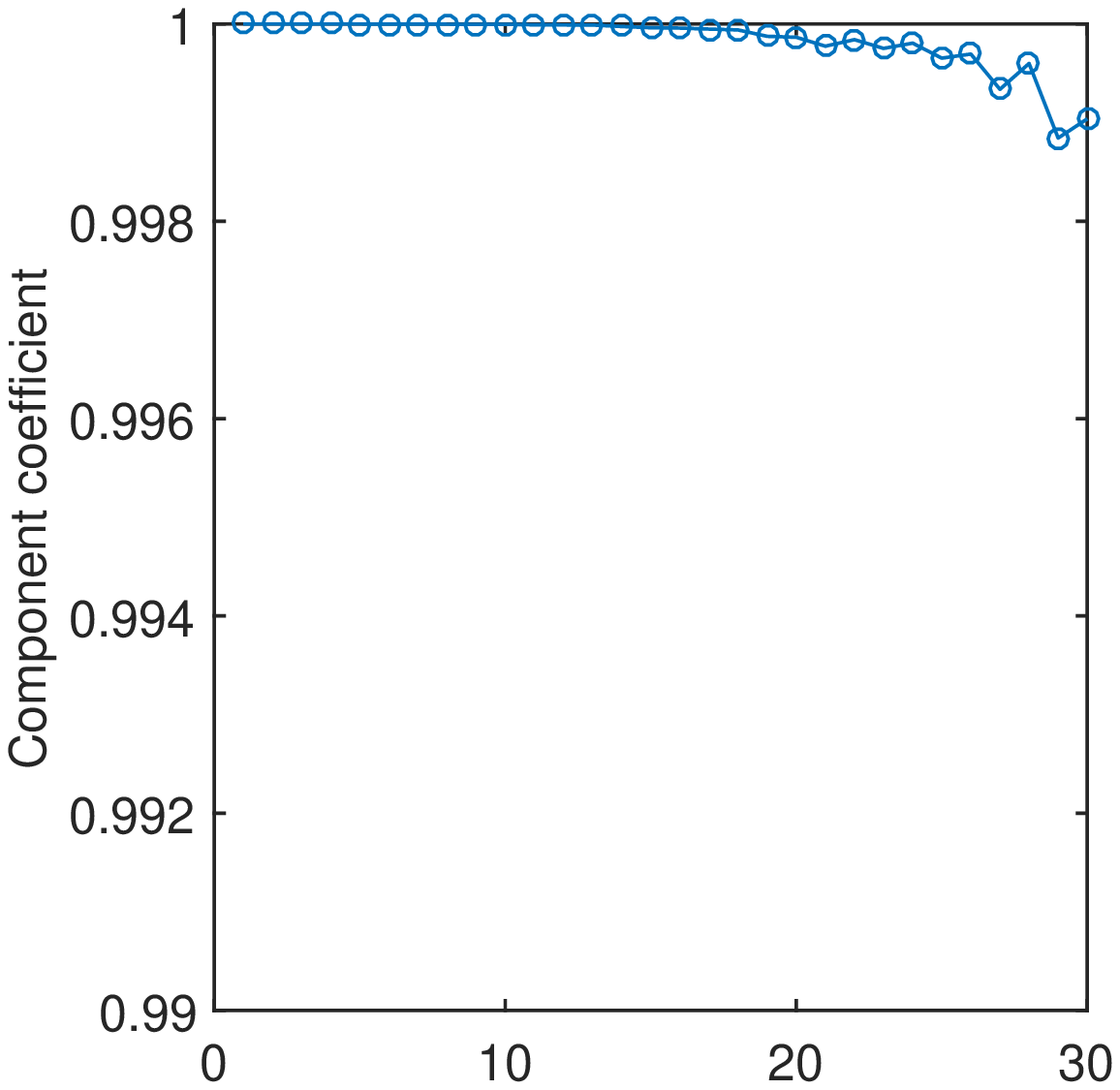}}
  \caption{The accuracy of Alg. 6 ($q=11$) on principal components of Matrix 2 (with comparison to the results from \texttt{svds}). (a) The numeric values (sorted) of first principal component.  (b) The correlation coefficients for the first 30 principal components.}
  \label {Fig3}
\end{figure}

Table 1 shows that the speedup ratio of Alg. 6 increases when $nnz(\mathbf{A})$ decrease, no matter compared with \texttt{svds} or to Alg. 1. The proposed algorithm is up to \textbf{20X} faster than \texttt{svds} and \textbf{9.1X} faster than the basic rPCA algorithm (both achieved on Matrix 3). In Fig. 1, the curves of eigSVDs is not shown, as they are indistinguishable to those of \texttt{svds}. From the figure, we see that the randomized PCA algorithms are indistinguishable from \texttt{svds} at the first tens of singular values. Alg. 6 is also indistinguishable from Alg. 1. This validates the effectiveness of the proposed algorithm with odd number of passes over $\mathbf{A}$. 

Fig. 2(a) shows the first principal component (i.e. $\mathbf{u}_1$) of Matrix 2 computed by \texttt{svds} and Alg. 6, which looks indistinguishable (only $1.4\times10^{-10}$ difference in $l_\infty$-norm). For other principal components, we calculate the correlation coefficient between the results obtained with the both methods. As shown in Fig. 2(b), the correlation coefficients are close to 1. The largest deviation occurs for the 29th principal component, with value 0.9988.

Secondly, we test the randomized algorithms with different $q$ and $p$ parameters. $q=2,4,6,9,11$  and $p=0,1,2,4,5$ are set to Alg. 6 and Alg. 1 respectively. The runtimes of the both algorithms are listed in Table 2, for computing the first 100 principal components of Matrix 2.
\begin{table}[h]
 \setlength{\abovecaptionskip}{0.05 cm}
 \caption{The runtimes of the basic rPCA algorithm and the proposed algorithm with different $q$ values.}
  \label{tab:table2}
  \centering
\small{
\begin{spacing}{0.9}
\renewcommand{\multirowsetup}{\centering}
  \begin{tabular}{@{}c@{~~}c@{~~}c@{~~}c@{~~}c@{~~}c@{~~}c@{~~}c@{~~}c@{~~}c@{~~}c@{~~}c@{}} 
  \toprule
    \multirow{2}{*}{Algorithm}  &
  \multicolumn{2}{c}{$q=2$} & \multicolumn{2}{c}{$q=4$} &  \multicolumn{2}{c}{$q=6$}  &\multicolumn{2}{c}{$q=9$} &\multicolumn{2}{c}{$q=11$}\\
  \cmidrule(r){2-3} \cmidrule(r){4-5} \cmidrule(r){6-7} \cmidrule(r){8-9} \cmidrule(r){10-11}
    & time & Sp1 & time & Sp1 & time & Sp1 & time & Sp1 & time & Sp1\\
\midrule
 Alg.1 ($p=\lfloor (q-1)/2\rfloor$) & 2.03 & * & 3.70 & * & 5.32 & * & 8.48 & * & 10.0 & * \\
 Alg.6 & 0.69 & 2.9 & 0.94 & 3.9  & 1.04 & 5.1 & 1.32 & 6.4 & 1.58 & 6.3 \\
\bottomrule 
 \end{tabular}
 \end{spacing}
 }
\end{table}

From the table, we see that the speedup ratio increases with $q$. At the same time, we plot Fig. 3 to show the  curves of computed singular values. From it we see with $q$ or $p$ increasing the singular values approach to the accurate values. And, since the proposed algorithm allows odd number of passes over $\mathbf{A}$, it has better flexibility. If lower accuracy is allowed, the \emph{frPCAt} algorithm runs much faster. For example,  setting $q=4$, it's actually 27X faster than \texttt{svds}.

\begin{figure}[h]         
  \centering
  \setlength{\abovecaptionskip}{0.02 cm}
  \subfigure[frPCAt ($q=2,4,6,9,11$)] {\includegraphics[width=2in, height=2in]{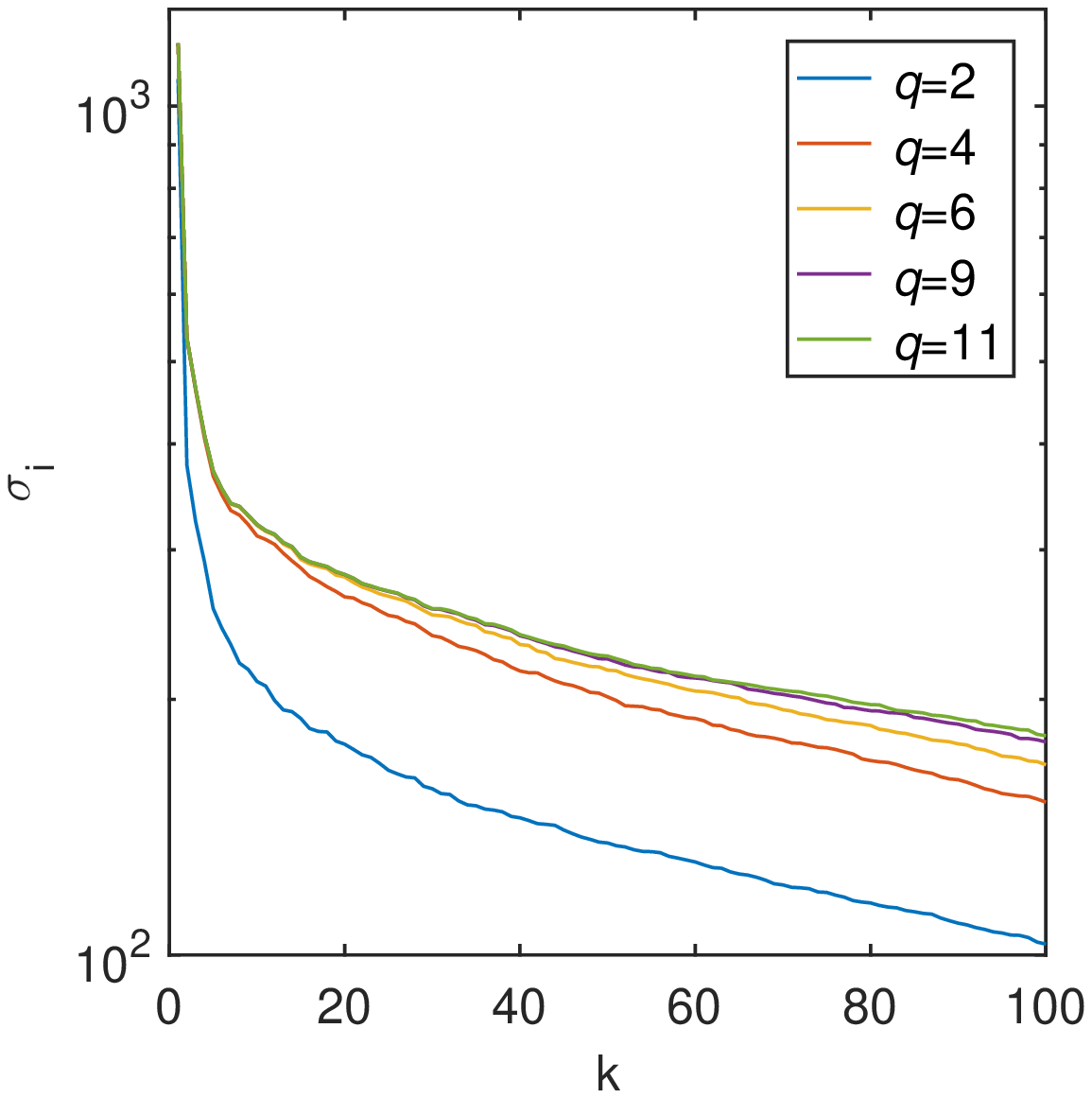}}
  \subfigure[basic rPCA ($p=0,1,2,4,5$)] {\includegraphics[width=2in, height=2in]{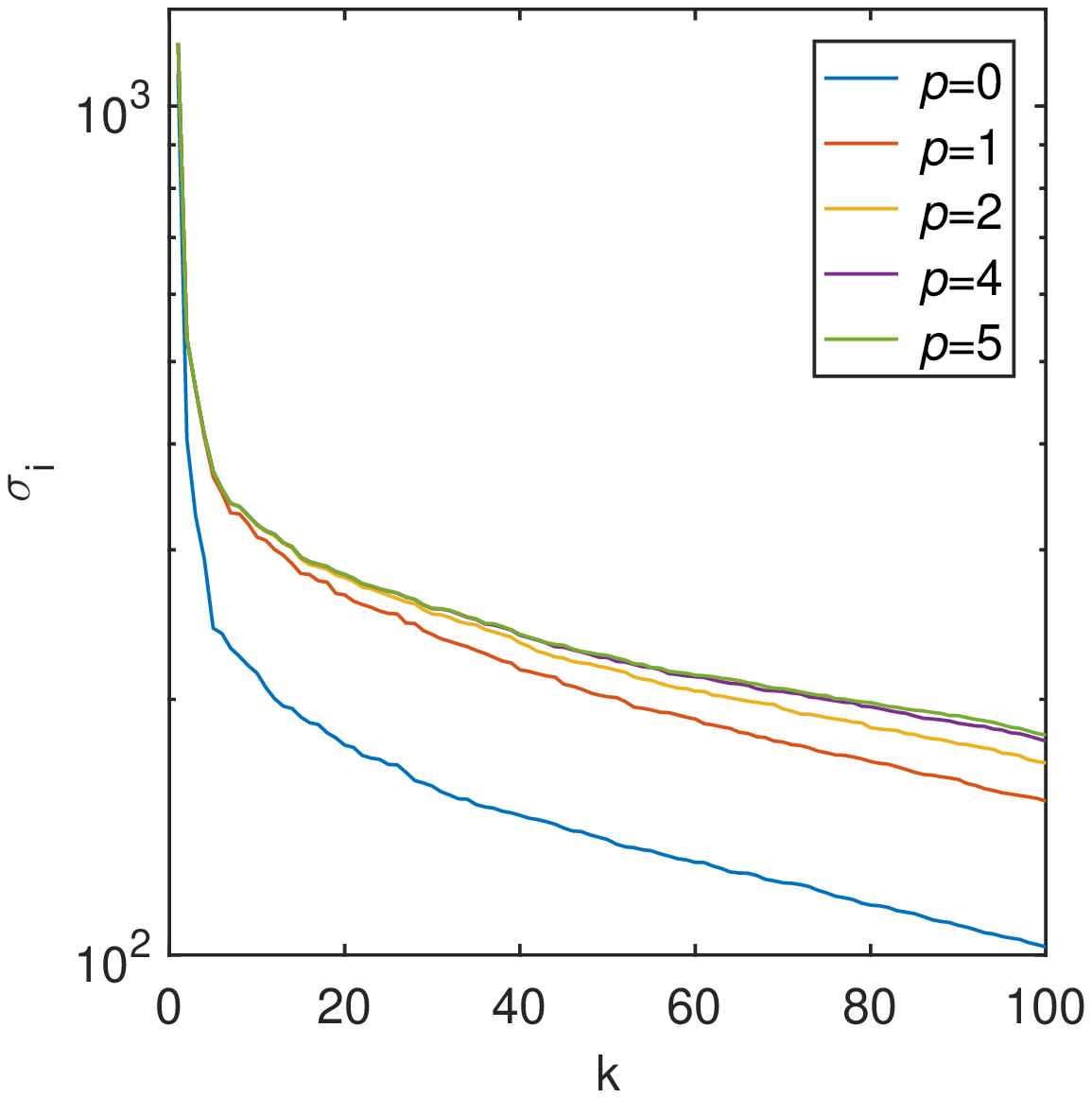}}
  \caption{The computed singular values of basic rPCA algorithm and proposed algorithm with different $q$ values ($p=\lfloor (q-1)/2\rfloor$).}
  \label {Fig2}
\end{figure}

Lastly, we construct matrices by modifying the dimensions of matrix. We only keep the first 107,966 columns of Matrix 6 to obtain Matrix 7  and the first 107,966 rows of Matrix 6 to get Matrix 8. For them we test different PCA algorithms. The results are listed in Table 3.
From it we see that eigSVDs (Alg. 3) is more efficient than \texttt{svds} when $m$ is much larger than $n$. And, Alg. 5 runs faster than Alg. 6, if when $m<n$. This validates the analysis in Section 3.3.
\begin{table}[h]
 \setlength{\abovecaptionskip}{0.05 cm}
 \caption{The runtimes of different PCA algorithms for matrices with different dimensions.}
  \label{tab:table2}
  \centering
\small{
\begin{spacing}{0.9}
\renewcommand{\multirowsetup}{\centering}
  \begin{tabular}{@{}c@{~~}c@{~~}c@{~~}c@{~~}c@{~~}c@{~~}c@{~~}c@{}} 
  \toprule
    \multirow{3}{*}{Algorithm}  &
  \multicolumn{2}{c}{Matrix 6} & \multicolumn{2}{c}{Matrix 7} & \multicolumn{2}{c}{Matrix 8} \\
  & \multicolumn{2}{c}{647,789$\times$323,896} & \multicolumn{2}{c}{647,989$\times$107,966} & \multicolumn{2}{c}{107,966$\times$323,896} \\
  \cmidrule(r){2-3} \cmidrule(r){4-5} \cmidrule(r){6-7}
    & time & Sp2 & time & Sp2  & time & Sp2  \\
\midrule
 svds     & 131.1& *  & 100.4 & *   &  51.3 & * \\
 eigSVDs & 75.7 & 1.7   & 16.1 & 6.3  & 47.0 & 1.1 \\
 Alg.1 ($p=5$)   & 84.5 & 1.6 & 59.6 & 1.7 & 35.6 & 1.4 \\
 Alg.5 ($q=11$) & 14.2 & 12 & 7.19 & 14  &  2.78 & \textbf{18}  \\
 Alg.6 ($q=11$)  & 10.2 & \textbf{13} & 5.91 & \textbf{18} & 3.79 & 14  \\
 \bottomrule 
 \end{tabular}
 \end{spacing}
 }
\end{table}

\subsection{Results on Real Large Data}
In this subsection, we test the proposed algorithm with three large real datasets. The first  one is a large matrix from MovieLens in size 270,896 $\times$ 45,115 with 97 nonzeros per row on average. The Aminer person-keyword matrix is the second dataset, in the largest size 12,869,521 $\times$ 323,899, with 16 nonzeros per row on average. The last one is a social network matrix from SNAP \citep{snapnets} in size 82,168 $\times$ 82,168 with 12 nonzeros per row on average. We computed the first 100 principal components/directions. The runtimes of different algorithms are listed in Table 4.


\begin{table}[h]
 \setlength{\abovecaptionskip}{0.05 cm}
 \caption{The runtimes of different PCA algorithms for real large sparse matrices.}
  \label{tab:table2}
  \centering
\small{
\begin{spacing}{0.9}
\renewcommand{\multirowsetup}{\centering}
  \begin{tabular}{@{}c@{~~}c@{~~}c@{~~}c@{~~}c@{~~}c@{~~}c@{~~}c@{}} 
  \toprule
    \multirow{2}{*}{Sparse Data}  &
  svds & eigSVDs & Alg. 1 ($p=5$) &\multicolumn{3}{c}{Alg. 6 ($q=12$)} \\
  \cmidrule(r){2-2} \cmidrule(r){3-3} \cmidrule(r){4-4} \cmidrule(r){5-7}
    & time & time & time & time & Sp1 & Sp2 \\
\midrule
MovieLens & 108.4 & 566.2 & 34.8 & 12.5 & 2.8 & 8.5\\
Aminer & * & * & 1448.3 & 398.7 & 3.6 & * \\
SNAP & 22.7 & 124.4 & 14.8 & 1.74 & 8.7 & 13 \\
\bottomrule 
 \end{tabular}
 \end{spacing}
 }
\end{table}

\begin{figure}[h]         
  \centering
  \setlength{\abovecaptionskip}{0.02 cm}
  \subfigure[MovieLens] {\includegraphics[width=1.8in, height=1.8in]{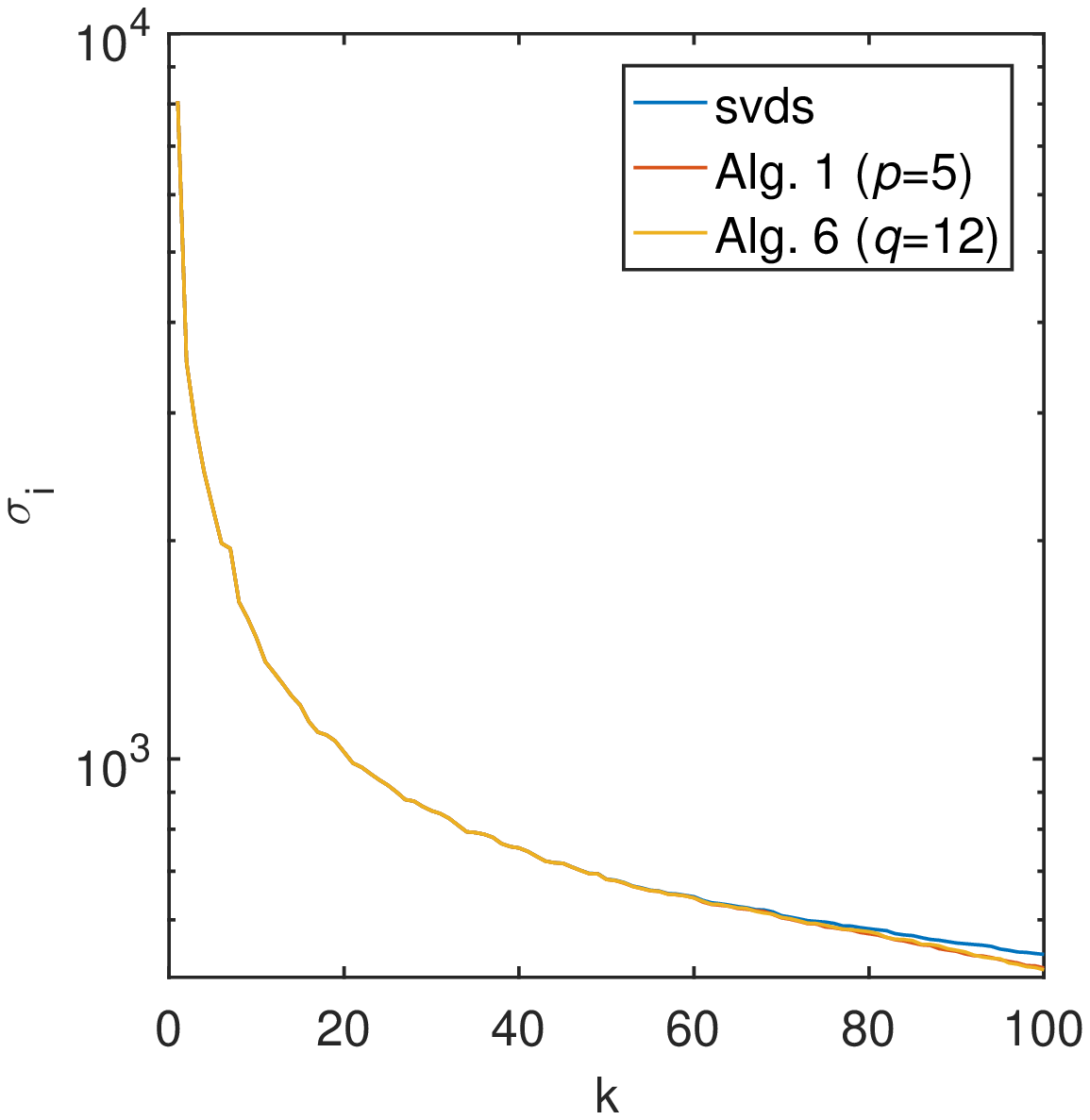}}
  \subfigure[Aminer] {\includegraphics[width=1.8in, height=1.8in]{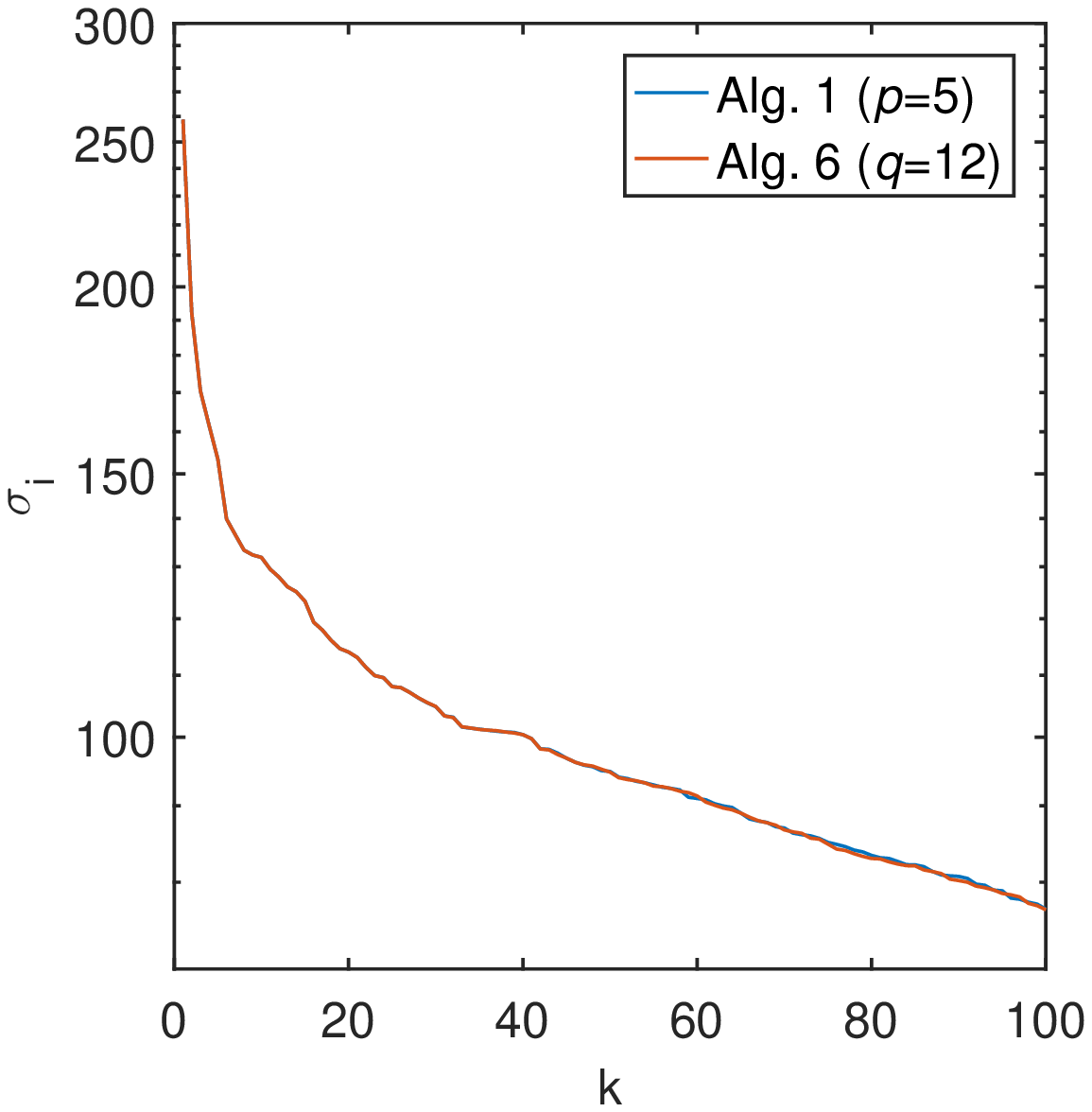}}
  \subfigure[SNAP] {\includegraphics[width=1.8in, height=1.8in]{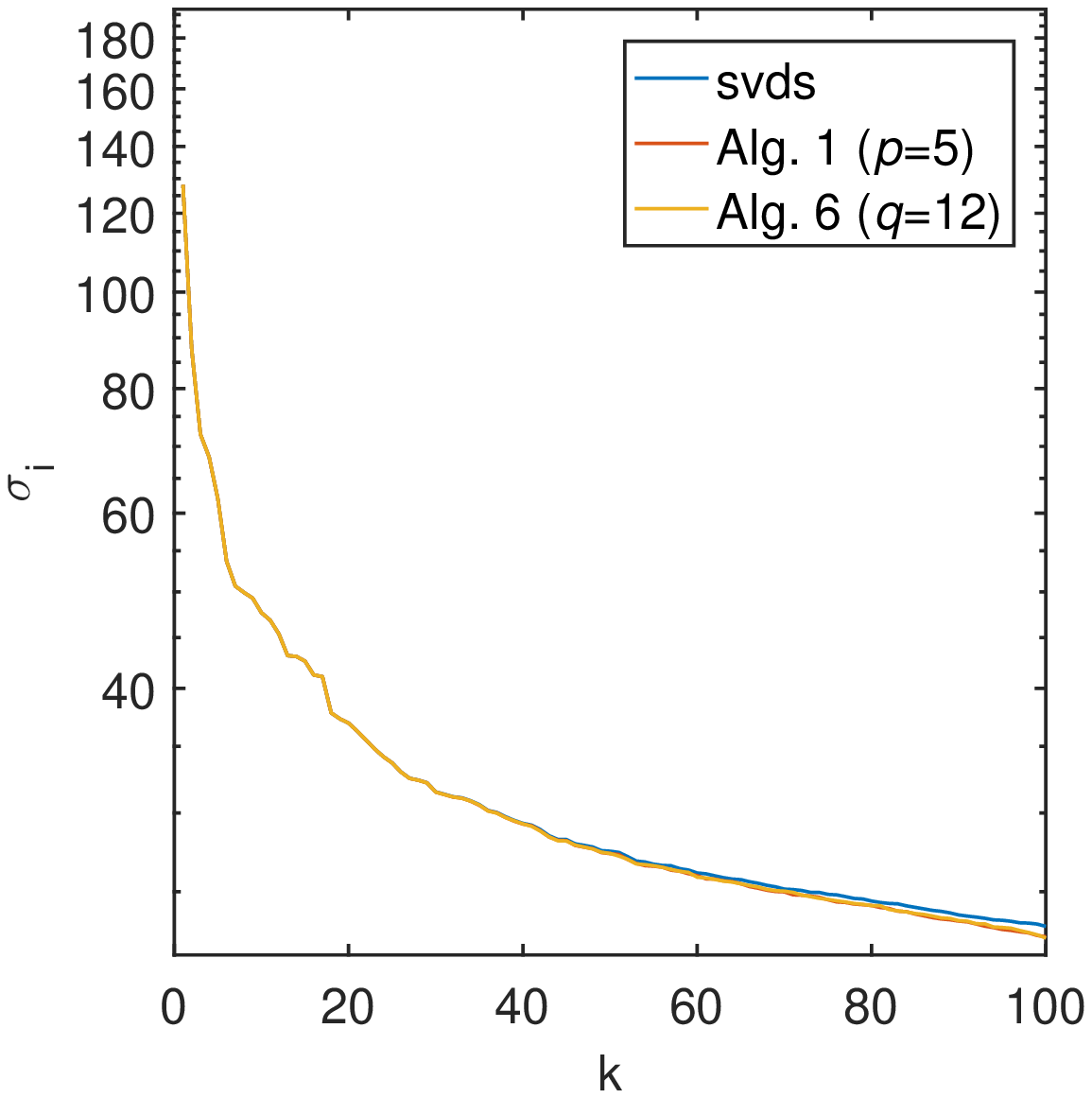}}
  \caption{The computed singular values of MovieLens, Aminer and SNAP matrix.}
  \label {Fig4}
\end{figure}

In the experiment, \texttt{svds} and eigSVDs failed on the Aminer dataset due to the limit of memory. The rPCA algorithms performed very well. The speedup ratio of Alg. 6 to  \texttt{svds} is up to 13X. Assuming that $C_{qr}=5C_{mul}$ in practice,  we can see the Sp1 on MovieLens and SNAP are 3.7 and 8.3 according to Eq. (19). They approximate the Sp1 in Table 4, which validates the analysis in Theorem 3. 
In Figure 4, we plot the computed singular values, showing the good accuracy of the proposed algorithm.  The memory costs of \texttt{svds}, Alg. 1 and Alg. 6 are 1.1 GB, 1.0 GB and 0.87 GB on MovieLens and 0.58 GB, 0.55 GB and 0.35 GB on SNAP, respectively. They suggest that the rSVD algorithms need less memory than \texttt{svds}. For the largest dataset, Aminer, the memory costs of Alg. 1 and Alg. 6 are 25 GB and 23GB, respectively, while \texttt{svds} fails due to out-of-memory.

\section{Conclusions}
A fast randomized PCA algorithm including several techniques is proposed for sparse matrix. It is faster than \texttt{svds} and the basic rPCA algorithm. Its speedup ratio is up to 20X to \texttt{svds} and 9.1X to the basic rPCA algorithm. 
On real data from information retrieval, recommender system and network analysis, the proposed frPCA algorithm performs well, while \texttt{svds} and eigSVDs algorithm may fail due to large memory cost. The frPCA algorithm runs up to 13X faster than \texttt{svds} and 8.7X faster than basic rPCA algorithm for the network analysis dataset with little accuracy loss.
\bibliography{feng18}

\end{document}